\documentclass[letterpaper, 10 pt, conference]{ieeeconf}  

\IEEEoverridecommandlockouts                              

\PassOptionsToPackage{table,RGB,dvipsnames}{xcolor}

\makeatletter
\let\NAT@parse\undefined
\makeatother

\usepackage{times}

\newif\ifarxiv
\arxivtrue

\ifarxiv
\usepackage[backend=biber,
    bibstyle=ieee,
    citestyle=numeric,
    sorting=none,
    natbib=true,
    url=true,
    maxbibnames=99
    ]{biblatex}
\else
\usepackage[backend=biber,
    style=numeric-comp,
    maxnames=1,
    minnames=1,
    sorting=none,
    natbib=true,
    url=false
    ]{biblatex}

    \DeclareNameAlias{default}{labelname}
    \renewbibmacro{in:}{}
\fi

\usepackage{multicol}
\usepackage[bookmarks=true]{hyperref}
\usepackage{cuted}
\usepackage{afterpage}

\let\labelindent\relax

\usepackage[utf8]{inputenc} \usepackage[T1]{fontenc}    \usepackage{hyperref}       \usepackage{nameref}
\usepackage{tocloft}
\usepackage{url}            \usepackage{booktabs}       \usepackage{amsfonts}       \usepackage{nicefrac}       \usepackage{microtype}      

\usepackage{xcolor}
\usepackage{xcolor-material}
\usepackage{colortbl}                   \usepackage{multicol}

\usepackage{graphicx} 
\usepackage{algorithm}
\usepackage{algpseudocode}
\usepackage{amsmath,amssymb}
\usepackage{amsthm}
\usepackage[font=small,skip=0pt]{caption}

\usepackage{dsfont}
\usepackage{soul}
\usepackage{enumitem}
\usepackage{subcaption}
\usepackage{microtype,wrapfig}
\usepackage{relsize}

\usepackage{dblfloatfix}

\usepackage{mathtools}
\usepackage{framed}
\usepackage{xspace}

\usepackage[short]{optidef}       

\usepackage{crossreftools}

\usepackage{placeins}

\usepackage{multirow}
\selectcolormodel{natural}
\usepackage{ninecolors}
\selectcolormodel{rgb}

\usepackage{etoolbox}

\usepackage{enumitem}
\usepackage{braket}                 

\usepackage{bm}

\usepackage{siunitx}
\usepackage{thm-restate}
\usepackage[framemethod=TikZ]{mdframed}
\usepackage[most]{tcolorbox}
\usepackage{thm-restate} 
\usepackage[utf8]{inputenc} \usepackage[T1]{fontenc}    \usepackage{hyperref}       \usepackage{nameref}
\usepackage{tocloft}
\usepackage{url}            \usepackage{booktabs}       \usepackage{amsfonts}       \usepackage{nicefrac}       \usepackage{microtype}      

\usepackage{xcolor}
\usepackage{xcolor-material}
\usepackage{colortbl}                   \usepackage{multicol}

\usepackage{graphicx} 
\usepackage{algorithm}
\usepackage{algpseudocode}
\usepackage{amsmath,amssymb}
\usepackage{amsthm}
\usepackage[font=small,skip=0pt]{caption}

\usepackage{dsfont}
\usepackage{soul}
\usepackage{enumitem}
\usepackage{subcaption}
\usepackage{microtype,wrapfig}
\usepackage{relsize}

\usepackage{dblfloatfix}

\usepackage{mathtools}
\usepackage{framed}
\usepackage{xspace}

\usepackage[short]{optidef}       

\usepackage{crossreftools}

\usepackage{placeins}

\usepackage{multirow}
\selectcolormodel{natural}
\usepackage{ninecolors}
\selectcolormodel{rgb}

\usepackage{etoolbox}

\usepackage{enumitem}
\usepackage{braket}                 

\usepackage{bm}

\usepackage{siunitx}
\usepackage{thm-restate}
\usepackage[framemethod=TikZ]{mdframed}
\usepackage[most]{tcolorbox}
\usepackage{thm-restate} 
\usepackage{xparse}

\usepackage{etoolbox}
\AfterEndEnvironment{tcolorbox}{\noindent\ignorespaces}

\declaretheoremstyle[bodyfont=\normalfont]{normalstyle}

\DeclareFontFamily{U}{stix2bb}{}
\DeclareFontShape{U}{stix2bb}{m}{n} {<-> stix2-mathbb}{}

\NewDocumentCommand{\ind}{}{\text{\usefont{U}{stix2bb}{m}{n}1}}

\declaretheorem[name=Theorem]{thm}
\declaretheorem[name=Lemma]{lemma}
\declaretheorem[name=Definition]{defi}

\declaretheorem[name=Counterexample,style=normalstyle]{counterexample}

\declaretheorem[name=Corollary]{coroll}
\declaretheorem[name=Remark,style=normalstyle]{rmk}
\declaretheorem[name=Note,numbered=no,style=normalstyle]{note}
\declaretheorem[name=Assumption]{assmp}

\definecolor{MaterialBlueGray10}{HTML}{CFD8DC}
\definecolor{MaterialBlueGray900}{HTML}{263238}
\definecolor{MaterialBlue10}{HTML}{E3F2FD}
\definecolor{MaterialBlue900}{HTML}{0D47A1}

\tcbset{
  ThmFrame/.style={
    enhanced,
    breakable,
    colback=MaterialBlue10!20,
    colframe=MaterialBlue900!55,
    boxrule=0pt,
    borderline west={1.3pt}{0pt}{MaterialBlue900!65},
    boxsep=7pt,
    top=1.5pt,
    arc=1pt,
    outer arc=1pt,
    bottom=3pt,
    sharp corners=all
  }
}

\tcbset{
  BaseFrame00/.style={
    enhanced,
    breakable,
    boxrule=0pt,
    boxsep=2pt,
    left=5pt,
    top=0pt,
    arc=1pt,
    outer arc=1pt,
    bottom=0pt,
    sharp corners=all,
width=1.03\linewidth,
    enlarge left by=-7pt,
    before upper={\setlength{\parindent}{10pt}}
  },
  BaseFrame0/.style={
    BaseFrame00,
    width=1.03\linewidth,
  },
  BaseFrameRight0/.style={
    BaseFrame00,
    width=1.06\linewidth,
  },
  BaseFrame/.style={
    BaseFrame0,
    before skip=\medskipamount,
    after skip=\medskipamount,
  },
  BaseFrameRight/.style={
    BaseFrameRight0,
    before skip=\medskipamount,
    after skip=\medskipamount,
  },
  BaseFrame3/.style={
    BaseFrame0,
    before skip=\smallskipamount,
    after skip=\smallskipamount,
  },
  BaseFrame4/.style={
    BaseFrame0,
    before skip=\smallskipamount,
    after skip=\medskipamount,
  },
}

\tcbset{
  ThmFrame2/.style={
    BaseFrame,
    colback=MaterialBlue50!40,
    colframe=MaterialBlue900!10,
    borderline west={2.0pt}{0pt}{MaterialBlue900!15},
  },
  ThmFrame3/.style={
    BaseFrame3,
    colback=MaterialBlue50!40,
    colframe=MaterialBlue900!10,
    borderline west={2.0pt}{0pt}{MaterialBlue900!15},
  },
DefFrame2/.style={
    BaseFrame,
    colback=MaterialBlueGray10!15,
    colframe=MaterialBlueGray900!10,
    borderline west={2.0pt}{0pt}{MaterialBlueGray900!20},
  },
  DefFrame3/.style={
    BaseFrame3,
    colback=MaterialGreen50!30,
    colframe=MaterialGreen900!10,
    borderline west={2.0pt}{0pt}{MaterialGreen900!15},
  },
AssmpFrame2/.style={
    BaseFrame,
    colback=MaterialBlueGray10!05,
    colframe=MaterialBlueGray900!10,
    borderline west={2.0pt}{0pt}{MaterialBlueGray900!05},
  },
NoteFrame2/.style={
    BaseFrame,
    colback=MaterialGreen50!30,
    colframe=MaterialGreen900!10,
    borderline west={2.0pt}{0pt}{MaterialGreen900!15},
  },
  NoteFrame3/.style={
    BaseFrame3,
    colback=MaterialYellow50!80,
    colframe=MaterialYellow900!10,
    borderline west={2.0pt}{0pt}{MaterialYellow900!30},
  },
  NoteFrame4/.style={
    BaseFrame4,
    colback=MaterialYellow50!80,
    colframe=MaterialYellow900!10,
    borderline west={2.0pt}{0pt}{MaterialYellow900!30},
  },
ProofFrame2/.style={
    BaseFrame,
    colback=MaterialBrown50!40,
    colframe=MaterialBrown900!10,
    borderline west={2.0pt}{0pt}{MaterialBrown900!25},
  },
  ProofFrameRight2/.style={
    BaseFrameRight,
    colback=MaterialBrown50!40,
    colframe=MaterialBrown900!10,
    borderline west={2.0pt}{0pt}{MaterialBrown900!25},
  },
  ProofFrame3/.style={
    BaseFrame3,
    colback=MaterialBrown50!40,
    colframe=MaterialBrown900!10,
    borderline west={2.0pt}{0pt}{MaterialBrown900!25},
  }
}

\definecolor{color_worksite}{rgb}{0.890,0.773,0.341}
\definecolor{color_gear}{rgb}{0.192,0.420,0.812}
\definecolor{color_wood}{rgb}{0.298,0.663,0.380}
\definecolor{color_saw}{rgb}{0.780,0.170,0.170}
\definecolor{color_door}{rgb}{0.050,0.050,0.050}
\definecolor{color_key}{rgb}{0.867,0.518,0.325}
\definecolor{color_safe}{rgb}{0.050,0.050,0.050}
\definecolor{color_coffee}{rgb}{0.612,0.513,0.435}
\definecolor{color_tea}{rgb}{0.545,0.443,0.702}

\protected\def\verythinspace{\ifmmode
    \mskip0.5\thinmuskip
  \else
    \ifhmode
      \kern0.08334em
    \fi
  \fi
}

\DeclareMathOperator*{\argmax}{\scalebox{0.8}{$\mathrm{arg\verythinspace max}$}}

\newcommand*{\Vopt}[1]{V^*_{[#1]}}
\newcommand*{\Qopt}[1]{Q_{[#1]}}
\newcommand*{\Vhist}[1]{\mathbf{V}^*_{[#1]}}
\newcommand*{\Qhist}[1]{\mathbf{Q}_{[#1]}}

\newcommand*{\witness}[1]{\sigma_{[#1]}}
\newcommand*{\optwitness}[1]{\sigma^*_{[#1]}}

\newcommand*{\timer}{\mathtt{t}}

\newcommand{\smax}[2][0.85]{\mathop{\scalebox{#1}{$\displaystyle \max\limits_{#2}$}}}
\newcommand{\smin}[2][0.85]{\mathop{\scalebox{#1}{$\displaystyle \min\limits_{#2}$}}}
\newcommand{\sminraise}[2][0.85]{\mathop{\raisebox{.5ex}{\scalebox{#1}{$\displaystyle \min\limits_{#2}$}}}}

\let\max\relax
\DeclareMathOperator*{\max}{\scalebox{0.8}{$\mathrm{max}$}}

\let\min\relax
\DeclareMathOperator*{\min}{\scalebox{0.8}{$\mathrm{min}$}}

\newcommand\numberthis{\addtocounter{equation}{1}\tag{\theequation}}

\newcommand\addtag[1][]{\refstepcounter{equation}\hfill(\theequation)\notblank{#1}{\label{#1}}{}}

\newcommand*{\LTLSet}{\Psi}
\newcommand{\LTLNext}{\mathsf{X}}
\newcommand{\LTLUntil}{\mathsf{U}}
\newcommand{\LTLFinally}{\mathsf{F}}

\newcommand{\LTLGlobally}{\mathsf{G}}

\newcommand{\IGUSet}{\mathcal{I}}

\newcommand{\rgu}{\mathsf{r}}

\newcommand{\IUSet}{\mathcal{J}}

\newcommand{\calX}{\mathcal{X}}
\newcommand{\calA}{\mathcal{A}}

\usepackage[defer]{deferproof}

\usepackage{cleveref}

\crefname{thm}{thm}{thms}
\Crefname{thm}{Thm.}{Thm.}

\crefname{lemma}{lem}{lem}
\Crefname{lemma}{Lem.}{Lem.}

\crefname{assmp}{asm.}{asm.}
\Crefname{assmp}{Asm.}{Asm.}

\crefname{defi}{def}{def}
\Crefname{defi}{Def.}{Def.}

\crefname{coroll}{cor}{cor}
\Crefname{coroll}{Cor.}{Cor.}

\crefname{rmk}{rmk}{rmk}
\Crefname{rmk}{Rmk.}{Rmk.}

\crefname{counterexample}{counterexample}{counterexamples}
\Crefname{counterexample}{Counterexample}{Counterexamples}

\crefname{section}{Sec.}{Secs.}
\Crefname{section}{Sec.}{Secs.}

\newif\ifcdc
\cdctrue

\title{\LARGE \bf
Value Functions for Temporal Logic: Optimal Policies and Safety Filters
}

\author{Oswin~So$^{*,1}$, William~Sharpless$^{*,2}$, Sylvia~Herbert$^2$, Chuchu~Fan$^{1}$\thanks{$^{1}$Oswin~So and Chuchu~Fan are with the Massachusetts Institute of Technology, Cambridge, MA 02139, USA.
      {\tt\small \{oswinso, chuchu\}@mit.edu}}\thanks{$^{2}$William~Sharpless and Sylvia~Herbert are with the University of California, San Diego, CA 92093, USA.
      {\tt\small \{wsharpless, sherbert\}@ucsd.edu}}}

\begin{document}
\definecolor{hlcolor}{HTML}{348ABD}

\maketitle

\thispagestyle{plain}
\pagestyle{plain}

\begin{abstract}
While Bellman equations for basic reach, avoid, and reach-avoid problems are well studied,
the relationship between value optimality and policy optimality becomes subtle in the undiscounted infinite-horizon setting, particularly for more complicated tasks.
Greedily maximizing the Q-function can produce policies that indefinitely defer task completion for reach-avoid problems, or equivalently, Until specifications, even when the value function is optimal.
Building upon recent results decomposing the value function for temporal logic (TL) into a graph of constituent value functions,
we construct non-Markovian policies based on state history that avoid this pathology and prove their optimality with respect to the quantitative robustness score for nested Until, Globally, and Globally-Until specifications. 
We further show how the Q function can serve as a safety filter for complex TL specifications, extending prior results beyond simple avoid or reach-avoid tasks.
\end{abstract}

\section{Introduction and Related Work}
The value function is central to optimal control and reinforcement learning (RL), encoding the optimal objective over action sequences. While it can be computed via dynamic programming (DP), this quickly becomes intractable in high dimensions. Instead, RL leverages the Bellman equation (BE), a recursive characterization of the value that enables approximation without grid-based DP. The associated Q-function further allows optimal policies to be obtained via single-step maximization. This framework is typically applied when objectives are defined by temporal sums of rewards (or costs), as in Markov Decision Processes (MDPs).

While this approach has been successful, the additive combination of rewards and costs in RL can be difficult; the optimal trajectory may incur significant cost so long as it obtains high reward, leading practitioners to iteratively tune hyperparameters until desirable performance is observed. To overcome this limitation, values based on temporal maxima and minima have been introduced from the field of Hamilton-Jacobi reachability (HJR) \cite{mitchell2005time,fisac2015reach}. In recent years, several works have integrated the max-min BE of HJR into RL frameworks to design safe and performant algorithms \cite{fisac2019bridging, hsu2021safety, Ganai2023, so2024solving}. 

\begin{figure}[t!]
    \definecolor{TRed}{HTML}{FF3926}
    \definecolor{TYellow}{HTML}{daa415}
    \definecolor{TBlue}{HTML}{32A7F5}
\centering
\includegraphics[width=\linewidth]{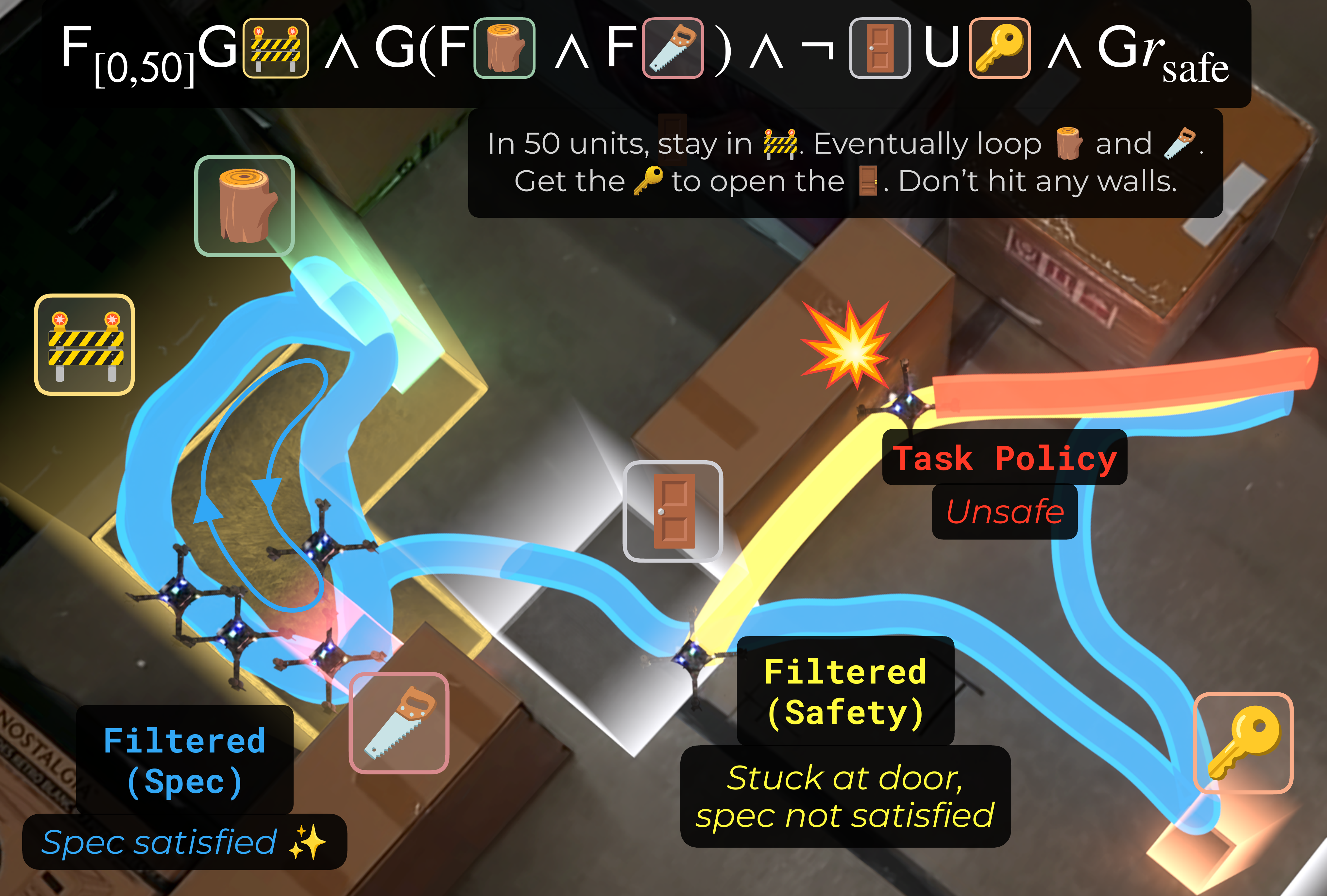}
   \vspace{-1.7ex}
   \caption{
   Hardware rollouts from a Crazyflie drone using an \textcolor{TRed}{unsafe task policy},
   the task policy \textcolor{TYellow}{filtered for safety}, and the task policy \textcolor{TBlue}{filtered for the specification}.
   \textcolor{TYellow}{Filtering for safety} maintains safety but violates the specification.
   Our proposed \textcolor{TBlue}{safety filter for TL specifications} guarantees the specification is satisfied.
}
\label{fig:hw_rollouts}
   \vspace{-2em}
\end{figure}

HJR characterizes values for basic safety and liveness tasks, called the reach, avoid and reach-avoid problems \cite{bansal2017hamilton}. 
Notably, the objectives of these values are equivalent to the quantitative semantics of basic predicates in Temporal Logic (TL) \cite{pnueli1977temporal, donze2010robust, fainekos2009robustness, LTL-and-beyond}, namely eventually ($\mathsf{F}$), always ($\mathsf{G}$), and until ($\mathsf{U}$) \cite{chen2018signal}. 
TL itself is merely a specification language but offers an attractive way to automatically translate and reduce complex task specifications into values for RL. 

Recently, it was demonstrated that HJR and the HJ-RL algorithms can be extended to broader classes of TL via algebraic decomposition of the value into a graph of values \cite{sharpless2025dual,sharpless2026bellman}. 
This was first considered in \cite{sharpless2025dual}, where the dual-objective values, called the reach-reach ($\mathsf{F} \land \mathsf{F}$) and reach-always-avoid ($\mathsf{F} \land \mathsf{G}$) problems, proved to decompose into coupled sets of BEs, and later extended to a broad class of TL specifications in \cite{sharpless2026bellman}. 
These methods scale the approximation of the value in high-dimension and model free settings to complex task specifications.

However, a novel challenge arises in these settings: simple single-step maximization of the corresponding Q function does not necessarily yield the optimal action sequence for infinite-horizon trajectories (see Ex.~\ref{ex:greedy_Q_suboptimal}), a fundamentally non-Markovian problem \cite{sharpless2025dual}. 
While decomposition allows one to solve the value for infinite trajectories, it remains to show when the optimal inputs must ``switch'' from single-step maximization of the current value to a constituent value in the graph, corresponding to the remaining unsatisfied  
and infinite-horizon specifications. 

In \cite{sharpless2025dual}, a Markovian policy is derived with an efficient state-augmentation that tracks relevant quantitative semantics and yields the optimal switching condition. However, this approach is derived in the dual-objectives and is difficult to generalize to complex formulae, particularly those that include reach-avoid-loop specifications ($\mathsf{GU}$). 
In this work, we propose a generalization of the value and Q-function to state histories to compute non-Markovian policies optimal with respect to any state history. 
This approach extends to the broad class of TL specifications for which the value may be decomposed \cite{sharpless2026bellman}. Moreover, this approach offers the novel ability to synthesize safety filters for a user-defined nominal policy that (1) guarantee specification satisfaction in a least-restrictive manner and (2) do so in a manner that is optimal given the prior sub-optimal path of the nominal policy. 

In this work, the optimality of the policy over histories and the guarantee of satisfaction are made possible by augmenting the system with a temporal state (Def.~\ref{def:timeaugsys}). For completeness, we re-derive the policies in \cite{sharpless2025dual} for the Until predicate via the novel framework and we show the approach proposed in this work extends to the larger class of formulae defined in Def.~\ref{def:solvable_fragment}. 
Ultimately, this work offers a practical approach to generate an optimal policy for the value of a broad class of TL formulae, and demonstrates the result may be used to filter a nominal policy to ensure complex task satisfaction.

\textbf{The contributions of this work are as follows:}
\begin{enumerate}[leftmargin=*]
\item We generalize the optimal value function and Q function (Def.~\ref{def:bellmanvalue}) and optimal policy (Def.~\ref{def:def_of_optimal_policy}) for state histories to handle the non-Markovian nature of the problem.
    \item For a broad class of TL formulae (Def.\ref{def:solvable_fragment}), we use this framework to construct a compositional optimal policy for the associated robustness metric (Thms.~\ref{thm:simple_until_policy_optimal}-\ref{thm:solvable_fragment}).
\item Using the Q function over (sub-optimal) history yields a least-restrictive safety filter for a nominal policy to ensure satisfaction (Thms.~\ref{thm:safety_filter},\ref{thm:safety_filter_satisfy}).
\end{enumerate}

\section{Preliminaries}
\subsection{Problem Setting and Notation} \label{sec:setup}
We consider a discrete-time system $x_{t+1} = f(x_t, a_t)$ with state $x_t \in \mathcal{X} \subseteq \mathbb{R}^n$ and action $a_t \in \mathcal{A} \subseteq \mathbb{R}^m$.
Let $x_{0:t} \in \mathcal{X}^{t+1}$ and $a_{0:t} \in \mathcal{A}^{t+1}$ denote state and action trajectories of length $t+1$, and use $x_t = x_{t:t} \in \mathcal{X}$ and $a_t = a_{t:t} \in \mathcal{A}$ for brevity.
Let $\mathcal{X}^+ \coloneqq \bigcup_{t=1}^\infty \mathcal{X}^t$ be the set of all finite-length state trajectories.
For a (non-Markovian) policy $\pi : \mathcal{X}^+ \to \mathcal{A}$, we say that the infinite trajectory $x^\pi_{0:\infty}$ is generated by $\pi$ from $x_{0:t} \in \mathcal{X}^+$ if $x^\pi_{0:t} = x_{0:t}$ and $x^\pi_{k+1} = f(x^\pi_k, \pi(x^\pi_{0:k}))$ for all $k \geq t$.
In general, we use the word ``policy'' to refer to a non-Markovian policy (dependent on history), and specify when it is Markovian
\footnote{One can always augment the state with the (infinite-dimensional) history to make the policy Markovian.
dHowever, this would be inefficient and we avoid this to distinguish the proposed Markovian approach}.

We use Linear Temporal Logic (LTL) \cite{baier2008principles} specifications augmented with quantitative semantics via the robustness score \cite{donze2010robust} from Signal Temporal Logic (STL) \cite{maler2004monitoring}.
LTL formulae comprise atomic propositions $p$ (Boolean variables that depend on the current state), logical operators ($\land$, $\lor$, $\lnot$) and temporal operators Until ($\LTLUntil$) and Next ($\LTLNext$), given by the following grammar in Backus-Naur Form:
\begin{equation}
    \phi \coloneqq \top \mid p \mid \neg \phi \mid \phi_1 \land \phi_2 \mid \phi_1 \LTLUntil \phi_2 \mid \LTLNext \phi,
\end{equation}
where $\top$ is the Boolean constant true. Other temporal operators can be derived from these, e.g., ``Finally'' $\LTLFinally$ and ``Globally'' $\LTLGlobally$.
Let $\LTLSet$ denote the set of all LTL formulae, and
let $\mathrm{Prop}$ denote the set of \emph{propositional formulae}, i.e., formulae that do not contain temporal operators.
We write propositional formulae in italics (e.g., $p$), while general formulae are in serif (e.g., $\mathsf{r}$).

We additionally associate every TL formula $\psi \in \LTLSet$ with a quantitative semantic called the \textit{robustness score} \cite{donze2010robust} $\rho_{[\psi]}: \mathcal{X}^\mathbb{N} \to \mathbb{R}$ that quantifies the satisfaction of $\psi$ by an infinite length state trajectory $x_{0:\infty} \in \mathcal{X}^\mathbb{N}$, with $\rho_{[\psi]}(x_{0:\infty}) \geq 0$ if and only if $x_{0:\infty}$ satisfies $\psi$.
Specifically, we associate every atomic proposition $p$ with a function $p : \mathcal{X} \to \mathbb{R}$ such that $p(x_0) = \rho_{[p]}(x_{0:\infty})$ for any trajectory $x_{0:\infty}$.
For brevity, \textbf{we abuse notation to refer to a formula and its robustness score interchangeably}, e.g., $\psi(x_{0:\infty}) \coloneqq \rho_{[\psi]}(x_{0:\infty})$, with the meaning clear from context.
$\rho$ is defined recursively as follows \cite{donze2010robust}:
\noindent\begin{minipage}{1.02\linewidth}
\vspace{\abovedisplayskip}
\centering
\hspace{-2.12ex}
\resizebox{1.02\linewidth}{!}{
\renewcommand{\arraystretch}{1.23}
\begin{tabular}{@{}ll@{}}
    $\rho_{[\neg \psi]}(x_{0:\infty}) \coloneqq -\psi(x_{0:\infty})$ & 
    $\rho_{[\psi_1 \land \psi_2]}(x_{0:\infty}) \coloneqq \psi_1(x_{0:\infty}) \land \psi_2(x_{0:\infty})$ \\
    $\rho_{[\LTLNext \psi]}(x_{0:\infty}) \coloneqq \psi(x_{1:\infty})$ &
    $\rho_{[\psi_1 \lor \psi_2]}(x_{0:\infty}) \coloneqq \psi_1(x_{0:\infty}) \lor \psi_2(x_{0:\infty})$ \\
    $\rho_{[\LTLGlobally \psi]}(x_{0:\infty}) \coloneqq \smin{t \geq 0} \psi(x_{t:\infty})$ &
    $\rho_{[\psi_1 \LTLUntil \psi_2]}(x_{0:\infty}) \coloneqq \smax{t \geq 0} \psi_2(x_{t:\infty}) \land \smin{0 \leq s < t} \psi_1(x_{s:\infty})$ \\
    $\rho_{[\LTLFinally \psi]}(x_{0:\infty}) \coloneqq \smax{t \geq 0} \psi(x_{t:\infty})$ &
    \hspace{18em} $\addtag[eq:robustness_def]$
\end{tabular}
}
\vspace{\belowdisplayskip}
\end{minipage}
\noindent where we have abused notation and used $\land$, $\lor$ to denote $\min$, $\max$ respectively.

\subsection{Optimal Control for Temporal Logic}

We now consider the problem of maximizing the robustness score of a given LTL formula $\psi \in \LTLSet$, i.e.,
\begin{equation} \label{eq:optimal_control_problem}
    \max_{a_{0:\infty}} \rho_{[\psi]}(x_{0:\infty}), \quad x_{t+1} = f(x_t, a_t),\; \forall t \geq 0.
\end{equation}
Following standard optimal control theory \cite{puterman2014markov}, we can define an optimal value function $V$ and Q function $Q$ for \eqref{eq:optimal_control_problem}:
\begin{tcolorbox}[DefFrame2]
    \begin{defi} \label{def:bellmanvalue}
        For a TL formula $\psi$ and $t \geq 0$, define the value function $\Vhist{\psi}: \mathcal{X}^+ \to \mathbb{R}$ (on histories) as
        \begin{equation} \label{eq:bellmanvalue}
            \Vhist{\psi}(x_{0:t}) \coloneqq \smax{a_{t:\infty}} \rho_{[\psi]}(x_{0:\infty}),
        \end{equation}
        and the Q function $\Qhist{\psi}: \mathcal{X}^+ \times \mathcal{A} \to \mathbb{R}$ (on histories) as
        \begin{equation} \label{eq:Q}
            \Qhist{\psi}(x_{0:t}, a_t) \coloneqq \smax{a_{t+1:\infty}} \rho_{[\psi]}(x_{0:\infty}),
        \end{equation}
        where $x_{k+1} = f(x_k, a_k)$ for $k \geq t$. This generalizes the conventional value function and Q function on states.
        In particular, for any $t \geq 0$, the value and Q function on states is recovered when the history has a length 1, i.e.,
        \begin{equation}
            \hspace{-.2em}
            \Vopt{\psi}(x_t) \coloneqq \Vhist{\psi}(x_{t:t}), \quad
            \Qopt{\psi}(x_t, a_t) \coloneqq \Qhist{\psi}(x_{t:t}, a_t).
        \end{equation}
    \end{defi}
\end{tcolorbox}
\begin{tcolorbox}[ThmFrame2]
\begin{lemma} \label{lem:Q_V_history}
    For any $\psi \in \LTLSet$, fix $\bar{x}_{0:k} \in \mathcal{X}^+$ for $k \geq 0$, and $\bar{a}_k \in \mathcal{A}$. Let $\bar{x}_{k+1} = f(\bar{x}_k, \bar{a}_k)$.
    Then,
    \begin{equation}
        \Qhist{\psi}(\bar{x}_{0:k}, \bar{a}_k) = \Vhist{\psi}(\bar{x}_{0:k+1}).
    \end{equation}
\end{lemma}
\end{tcolorbox}
To simplify the results, we assume the $\max$ in the definitions of $\Vhist{\psi}$, $\Qhist{\psi}$, and robustness scores \eqref{eq:robustness_def} are attainable.
This can be guaranteed with the following assumption:\footnote{
The assumptions can be relaxed by using $\sup$ instead of $\max$ 
in \eqref{eq:bellmanvalue} and \eqref{eq:Q},
but this yields $\epsilon$-optimal policies instead of optimal policies.
}

\begin{tcolorbox}[AssmpFrame2]
\begin{assmp} \label{assmp:max_attainable:finite}
The robustness score of all atomic propositions takes values in a finite set.
\end{assmp}
\end{tcolorbox}
This generalizes the finite state space assumption in \cite{sharpless2025dual} to additionally enable policy construction in infinite state spaces.
Crucially, the use of indicator functions for atomic propositions satisfies \Cref{assmp:max_attainable:finite}.

We next define optimality for our non-Markovian problem.
\begin{tcolorbox}[DefFrame2]
\begin{defi} \label{def:def_of_optimal_policy}
A policy $\pi : \mathcal{X}^+ \to \mathcal{A}$ is \textbf{optimal} (over histories) for LTL formula $\psi \in \LTLSet$ if, for all histories $x_{0:t} \in \mathcal{X}^+$ for $t \geq 0$, the trajectory $x^\pi_{0:\infty}$ generated by $\pi$ from $x_{0:t}$ achieves the maximal robustness score, i.e.,
   \vspace{-1.5ex}
   \begin{equation} \label{eq:optimality_eq}
      \rho_{[\psi]}(x^\pi_{0:\infty}) = \Vhist{\psi}(x_{0:t}).
   \end{equation}
\end{defi}
\end{tcolorbox}
\begin{tcolorbox}[NoteFrame2]
\begin{rmk}[Notions of Optimality]\label{rmk:optimality}
{
The notion of optimality (over histories) in \Cref{def:def_of_optimal_policy} is stronger than the conventional notion in MDPs, which only requires optimality over initial states as used in \cite{sharpless2025dual}.
The two notions of optimality coincide for Markovian problems \cite{puterman2014markov}, but can differ for non-Markovian problems.}
\end{rmk}
\end{tcolorbox}

In this work, we construct optimal policies for the following fragment of TL formulas defined recursively.
\let\origlabel\label
\begin{tcolorbox}[DefFrame2]
\let\label\origlabel
\begin{defi} \label{def:solvable_fragment}
    Define the set of solvable formulas $\mathcal{S}$ as the smallest set of TL formulas satisfying:
    \begin{enumerate}[label=\textbf{(S\arabic*)}, ref=\thedefi-(S\arabic*)]
        \item \label[defi]{def:S:Prop} \textbf{Propositional.}\quad If $p \in \mathrm{Prop}$, then $p \in \mathcal{S}$.
        \item \label[defi]{def:S:Until} \textbf{Until.}\quad If $p \in \mathrm{Prop}$ and $\varphi \in \mathcal{S}$, then $p \LTLUntil \varphi \in \mathcal{S}$.
        \item \label[defi]{def:S:Next} \textbf{Next.}\quad If $\varphi \in \mathcal{S}$, then $\LTLNext \varphi \in \mathcal{S}$.
        \item \label[defi]{def:S:Globally} \textbf{Globally.}\quad If $p \in \mathrm{Prop}$, then $\LTLGlobally p \in \mathcal{S}$.
        \item \label[defi]{def:S:GU} \textbf{Globally-Until.}\quad If $p_i, r_i \in \mathrm{Prop}$ for $i = 1, \dots, N$ with $N \geq 1$, then $\LTLGlobally\bigl( \bigwedge_{i=1}^N p_i \LTLUntil r_i \bigr) \in \mathcal{S}$.
        \item \label[defi]{def:S:Disjunction} \textbf{Disjunction.}\quad If $\varphi_1, \varphi_2 \in \mathcal{S}$, then $\varphi_1 \lor \varphi_2 \in \mathcal{S}$.
        \item \label[defi]{def:S:Conjunction} \textbf{Propositional conjunction.}\quad If $p \in \mathrm{Prop}$ and $\varphi \in \mathcal{S}$, then $p \land \varphi \in \mathcal{S}$.
    \end{enumerate}
\end{defi}
\end{tcolorbox}

We will construct an optimal policy for each item in \Cref{def:solvable_fragment}.
The case for \Cref{def:S:Prop} is straightforward, since $p \in \mathrm{Prop}$ does not depend on any actions.
In the following sections, we construct an optimal policy for remaining items in \Cref{def:solvable_fragment} before showing that our proposed policy optimally solves the entire fragment $\mathcal{S}$ (\Cref{thm:solvable_fragment}).

\section{Constructing an Optimal Policy for Until} \label{sec:until}
We first construct an optimal policy for the Until operator, $\LTLUntil$, as in \Cref{def:S:Until}. Until is a fundamental building block of LTL: every LTL formula can be rewritten using only the logical operators $\lnot$, $\lor$ and the temporal operators $\LTLUntil$ and $\LTLNext$ \cite{baier2008principles}.
Yet, constructing an optimal policy for an Until specification is not straightforward, as we show next.

In discounted MDPs, the policy that greedily maximizes the Q function is optimal \cite[Corollary 6.2.8]{puterman2014markov}.
However, as shown next, this is not true for Until specifications, as the greedy policy may indefinitely defer task completion,
an issue also identified in recent works \cite{li2025solving}.
\begin{tcolorbox}[DefFrame2]
\begin{counterexample} \label{ex:greedy_Q_suboptimal}
    Consider a finite state space $\mathcal{X} = \{ \mathsf{0}, \mathsf{1} \}$, and a finite action space $\mathcal{A} = \{ \mathsf{0}, \mathsf{1} \}$ with dynamics $f(x, a) = a$. Consider the formula $\psi = \LTLFinally \ind_\mathsf{1} \equiv \top \LTLUntil \ind_\mathsf{1}$, which requires the system to eventually reach state $\mathsf{1}$.
    Since state $\mathsf{1}$ can be reached in one step from any state, the Q function is equal to $1$ for all history-action pairs, i.e., for any $x_{0:t} \in \mathcal{X}^+$ and $a_t$,
        $\Qopt{\psi}(x_{0:t}, a_t) = 1$.
    In particular, for any history $x_{0:t}$,
    \begin{equation}
        \mathsf{0} \in \argmax_{a_t \in \mathcal{A}} \Qopt{\psi}(x_{0:t}, a_t).
    \end{equation}
    However, the policy $\pi(x_{0:k}) = \mathsf{0}$ for all $x_{0:k}$ does not satisfy the specification $\psi$, since it results in the trajectory $x_k = \mathsf{0}$ for all $k \geq 0$, which never reaches $\mathsf{1}$ and thus has robustness score $0$ despite $\Vhist{\psi}(x_{0:k}) = 1$ for all $x_{0:k}$.
\end{counterexample}
\end{tcolorbox}

To prevent indefinitely delaying score maximization,
we seek a time-optimal policy that achieves the maximal score in minimum time, which we next explore.
For the rest of \Cref{sec:until}, we consider $\psi \coloneqq q \LTLUntil \mathsf{r}$ for $q \in \mathrm{Prop}$, $\mathsf{r} \in \LTLSet$.

\subsection{Finite Horizon Until} \label{sec:until_finite_horizon}
{
Consider the finite-horizon versions of the Until operator (as in STL \cite{maler2004monitoring} or metric temporal logic \cite{alur1993real}), which does not suffer from this issue due to the finite horizon.
Choosing the shortest horizon that still recovers the maximal score forces completion as soon as possible.
}

\begin{tcolorbox}[DefFrame2]
\begin{defi}[Time Augmented System] \label{def:timeaugsys}
    Define the augmented state space $\tilde{\mathcal{X}} \coloneqq \mathcal{X} \times \mathbb{Z}$ with state
    $\tilde{x} = [x, \timer{}] \in \tilde{\mathcal{X}}$,
    where $\timer{} \in \mathbb{Z}$ is a timer that decreases every step, and the dynamics function $\tilde{f} : \tilde{\mathcal{X}} \times \mathcal{A} \to \tilde{\mathcal{X}}$ is given by
    \begin{equation}
        \tilde{f}([x, \timer{}], a) = [f(x, a), \timer{} - 1].
    \end{equation}
    For conciseness, we write the trajectory $\tilde{x}_{0:\infty}$ as $[x_{0:\infty}, \timer{}_0]$ for $\tilde{x}_0 = [x_0, \timer{}_0]$.
\end{defi}
\end{tcolorbox}
We use this timer to enforce the time limit within which a TL formula must be satisfied. The timer is initialized at $\timer{}_0$, and satisfaction must occur by $\timer{}=0$. We formalize this as $\tilde{\psi} \coloneqq q \LTLUntil (\mathsf{r} \land r_{\timer{}\geq0})$, where $r_{\timer{}\geq0}(x, \mathsf{t}) = \infty$ for $\mathsf{t} \geq 0$ and $-\infty$ otherwise.
Note that $\rho_{[\tilde{\psi}]}([x_{0:\infty}, \timer{}_0]) = -\infty$ if $\timer{}_0 < 0$.
In other words, $\mathsf{r}$ must be satisfied within the time limit specified by the timer to achieve a positive robustness score, which resembles a finite-horizon Until specification. We now make this connection precise.

\begin{tcolorbox}[ThmFrame2]
\begin{lemma} \label{lem:finite_horizon_semantics}
    $\tilde{\psi}$ has the same semantics as the timed Until operator $\LTLUntil_{[0, \timer{}_0]}$:
    \begin{equation}
        \rho_{[\tilde{\psi}]}([x_{0:\infty}, \timer{}_0])
        = \smax{0 \leq t \leq \timer{}_0} \min\{ \mathsf{r}(x_{t:\infty}), \smin{0 \leq k < t} q(x_k) \}.
    \end{equation}
    where the maximization over $t$ only considers the interval $[0, \timer{}_0]$ instead of $[0, \infty)$.
\end{lemma}
\end{tcolorbox}
\begin{tcolorbox}[AssmpFrame2]
\begin{proof}
    We prove this by induction on $\timer{}_0$. For $\timer{}_0=0$,
    \begin{align*}
        \rho_{[\tilde{\psi}]}([x_{0:\infty}, 0])
        &= \mathsf{r}(x_{0:\infty}) \lor \bigl( q(x_0) \land -\infty \bigr), \\
        &= \mathsf{r}(x_{0:\infty})
        = \smax{0 \leq t \leq 0} \min\{ \mathsf{r}(x_{t:\infty}), \smin{0 \leq k < t} q(x_k) \}.
    \end{align*}
    Now, assume the statement holds for $\timer{}_0 = n$. Then,
    \begin{align*}
        &\mathrel{\phantom{=}}\rho_{[\tilde{\psi}]}([x_{0:\infty}, n+1]) \\
        &= \mathsf{r}(x_{0:\infty}) \lor \bigl( q(x_0) \land \rho_{[\tilde{\psi}]}([x_{1:\infty}, n]) \bigr), \\
        &= \mathsf{r}(x_{0:\infty}) \lor \bigl( q(x_0) \land \smax{0 \leq t \leq n} \min\{ \mathsf{r}(x_{t+1:\infty}), \smin{1 \leq k < t+1} q(x_k) \} \bigr), \\
        &= \smax{0 \leq t \leq n+1} \min\{ \mathsf{r}(x_{t:\infty}), \smin{0 \leq k < t} q(x_k) \}.
    \end{align*}
    Therefore, by induction, it holds for all $\timer{}_0 \geq 0$.
\end{proof}
\end{tcolorbox}
Thus, for initial state $[x_0, \timer{}_0]$, we use the timed until operator $\LTLUntil_{[0, \timer{}_0]}$ as a shorthand such that $q \LTLUntil_{[0, \timer{}_0]} \mathsf{r}$ refers to the formula $q \LTLUntil (\mathsf{r} \land r_{\timer{} \geq 0})$.
Note that for any $x_{0:\infty}$ and $\timer{}_0$,
\begin{equation} \label{eq:finite_horizon_ineq}
   \rho_{[q \LTLUntil_{[0, \timer{}_0]} \mathsf{r}]}([x_{0:\infty}, \timer{}_0]) \leq \rho_{[q \LTLUntil \mathsf{r}]}(x_{0:\infty}).
\end{equation}
We next look at when the finite-horizon Until recovers the infinite-horizon Until. 
Specifically, we identify a ``witness time'' when the supremum in the definition of $\LTLUntil$ \eqref{eq:robustness_def} is attained.
\begin{tcolorbox}[ThmFrame2]
\begin{lemma} \label{lem:witness_exists}
    Under \Cref{assmp:max_attainable:finite}, for any $x_{0:\infty}$, there exists a smallest witness time $\tau = \witness{\psi}(x_{0:\infty}) \geq 0$ where
    \begin{align*}
        \rho_{[\psi]}(x_{0:\infty})
        &\coloneqq \sup_{t \geq 0} \color{hlcolor}{\min\{ \mathsf{r}(x_{t:\infty}), \smin{0 \leq k < t} q(x_k) \}} , \numberthis\label{eq:witness_tmp} \\
        &= \min\{ \mathsf{r}(x_{\tau:\infty}), \smin{0 \leq k < \tau} q(x_k) \}, \numberthis \label{eq:witness_exist:1}
    \end{align*}
    In other words, the supremum in the definition of $\LTLUntil$ is first attained at some finite time $\witness{\psi}(x_{0:\infty})$.
\end{lemma}
\end{tcolorbox}
\begin{tcolorbox}[AssmpFrame2]
\begin{proof}
    By \Cref{assmp:max_attainable:finite}, the robustness score takes values in a finite set.
    This is preserved under $\min$.
    Thus, the \textcolor{hlcolor}{blue term} in \eqref{eq:witness_tmp} also takes values in a finite set and attains its supremum at a finite time $\tau$. 
\end{proof}
\end{tcolorbox}
The above witness time holds for any trajectory.
We now define an optimal witness time, which is the smallest among all trajectories that achieve the optimal value.
\begin{tcolorbox}[DefFrame2]
\begin{defi} \label{def:optimal_witness}
   For $x_{0:t} \in \mathcal{X}^+$, define the optimal witness time $\optwitness{\psi} : \mathcal{X}^+ \to \mathbb{Z}_{\geq 0}$ as
   \begin{equation} \label{eq:optimal_witness}
      \optwitness{\psi}(x_{0:t}) = \min\{ \witness{\psi}(x_{0:\infty}) \mid \rho(x_{0:\infty}) = \Vhist{\psi}(x_{0:t}) \}.
   \end{equation}
   In other words, $\optwitness{\psi}(x_{0:t})$ is the smallest witness time among all trajectories starting at $x_{0:t}$ that achieve the optimal value $\Vhist{\psi}(x_{0:t})$.
\end{defi}
\end{tcolorbox}
Importantly, $\tilde{\psi}$ has the same value as $\psi$
at $\timer{}_0 = \optwitness{\psi}$.
\begin{tcolorbox}[ThmFrame2]
\begin{lemma} \label{lem:finite_infinite_horizon_Q}
   Let $\timer{}_s = \optwitness{\psi}(x_{0:s}) - s$ be the optimal witness time for $x_{0:s}$, and $\timer{}_0 = \timer{}_s + s = \optwitness{\psi}(x_{0:s})$. Then,
   \begin{equation} \label{eq:finite_infinite_horizon_Q}
      \hspace{-0.5em}
      \max_{a_t} \Qopt{\tilde{\psi}}([x_{0:s}, \timer{}_0], a_t) = \Vhist{\tilde{\psi}}([x_{0:s}, \timer{}_0]) = \Vhist{\psi}(x_{0:s}).
   \end{equation}
\end{lemma}
\end{tcolorbox}
\begin{condproof}[lem:finite_infinite_horizon_Q]
   By definition of $\optwitness{\psi}(x_{0:t})$ \eqref{eq:optimal_witness}, there exists $x_{0:\infty}$ where $\optwitness{\psi}(x_{0:s}) = \witness{\psi}(x_{0:\infty})$ and $\rho(x_{0:\infty}) = \Vopt{\psi}(x_{0:s})$.
   Using \eqref{eq:witness_exist:1},
   \begin{align*}
      \Vopt{\psi}(x_{0:s})
      &= \rho(x_{0:\infty}) \\
      &= \max_{0 \leq t \leq \timer{}_0} \min\{ \rho_{[\mathsf{r}]}(x_{t:\infty}), \smin{0 \leq k < t} q(x_k) \}, \\
      &\leq \smax{a_{s:\infty}} \smax{0 \leq t \leq \timer{}_0} \min\{ \rho_{[\mathsf{r}]}(x_{t:\infty}), \smin{0 \leq k < t} q(x_k) \}, \\
      &= \max_{a_s} \Qopt{\psi_{n^*}}([x_{0:s}, \timer{}_0], a_s)
      = \Vopt{\tilde{\psi}}([x_{0:s}, \timer{}_0]).
   \end{align*}
   Also, by \eqref{eq:finite_horizon_ineq}, $\Vopt{\psi}(x_{0:s}) \geq \Vopt{\tilde{\psi}}([x_{0:s}, \timer{}_0])$.
   Thus, \eqref{eq:finite_infinite_horizon_Q} holds.
\end{condproof}

We will show how relaxing the minimum-time requirement enables construction of a \textit{Markovian} optimal policy that achieves the maximal score without requiring keeping track of the timer state $\timer{}$ (\Cref{subsec:until_markovian_policy}).

\begin{condsection}[witness_persistence]
We now give a technical lemma on when the witness time remains unchanged as the history grows.
\begin{tcolorbox}[ThmFrame2]
\begin{lemma}[Conditional witness persistence]
\label{lem:abstract_witness_persistence}
Fix a formula $\chi \in \LTLSet$ and a policy $\pi:\mathcal{X}^+ \to \mathcal{A}$.
Let $\bar{x}_{0:\infty}$ be the trajectory generated by $\pi$ from
$\bar{x}_{0:s}$, and define
$ \tau \coloneqq \optwitness{\chi}(\bar{x}_{0:s}) $.
Assume that for every $k \ge s$, whenever $\optwitness{\chi}(\bar{x}_{0:k})=\tau$,
the following hold:
\begin{enumerate}[label=(\roman*), ref=\thelemma-(\roman*)]
    \item \label[lemma]{item:cond_witness_persist_value}
    $
    \Vhist{\chi}(\bar{x}_{0:k+1})=\Vhist{\chi}(\bar{x}_{0:k});
    $
    \item \label[lemma]{item:cond_witness_persist_attain}
    there exists a continuation $\hat{x}_{0:\infty}$ of $\bar{x}_{0:k+1}$
    such that
    \[
    \rho_{[\chi]}(\hat{x}_{0:\infty})=\Vhist{\chi}(\bar{x}_{0:k+1})
    \]
    with $ \witness{\chi}(\hat{x}_{0:\infty}) \le \tau $.
\end{enumerate}
Then
\[
\optwitness{\chi}(\bar{x}_{0:k})=\tau,
\qquad \forall k \ge s.
\]
\end{lemma}
\end{tcolorbox}
\begin{tcolorbox}[AssmpFrame2]
\begin{proof}
We prove the claim by induction on $k \ge s$.

\textbf{Base case.}
For $k=s$, the claim holds by the definition of $\tau$.

\textbf{Induction step.}
Assume that for some $k \ge s$,
\[
\optwitness{\chi}(\bar{x}_{0:k})=\tau.
\]
We show that
\[
\optwitness{\chi}(\bar{x}_{0:k+1})=\tau.
\]

By assumption \eqref{item:cond_witness_persist_attain}, there exists a continuation
$\hat{x}_{0:\infty}$ of $\bar{x}_{0:k+1}$ such that
\[
\rho_{[\chi]}(\hat{x}_{0:\infty})=\Vhist{\chi}(\bar{x}_{0:k+1})
\]
and
\[
\witness{\chi}(\hat{x}_{0:\infty}) \le \tau.
\]
Hence, by the definition of optimal witness time,
\begin{equation} \label{eq:cond_witness_persist:leq}
    \optwitness{\chi}(\bar{x}_{0:k+1}) \le \tau.
\end{equation}

Suppose, for contradiction, that
\[
\optwitness{\chi}(\bar{x}_{0:k+1}) < \tau.
\]
Then there exists a continuation $\tilde{x}_{0:\infty}$ of $\bar{x}_{0:k+1}$ such that
\[
\rho_{[\chi]}(\tilde{x}_{0:\infty})=\Vhist{\chi}(\bar{x}_{0:k+1})
\]
and
\[
\witness{\chi}(\tilde{x}_{0:\infty})
=
\optwitness{\chi}(\bar{x}_{0:k+1})
< \tau.
\]
Since $\tilde{x}_{0:\infty}$ is also a continuation of $\bar{x}_{0:k}$, and by
assumption \Cref{item:cond_witness_persist_value},
\[
\Vhist{\chi}(\bar{x}_{0:k+1})=\Vhist{\chi}(\bar{x}_{0:k}),
\]
the same continuation $\tilde{x}_{0:\infty}$ achieves $\Vhist{\chi}(\bar{x}_{0:k})$
with witness time strictly smaller than $\tau$.
This contradicts the inductive hypothesis
\[
\optwitness{\chi}(\bar{x}_{0:k})=\tau.
\]
Therefore,
\[
\optwitness{\chi}(\bar{x}_{0:k+1}) \ge \tau.
\]
Combined with \eqref{eq:cond_witness_persist:leq}, this gives
\[
\optwitness{\chi}(\bar{x}_{0:k+1})=\tau.
\]

Thus the claim holds for $k+1$, and the result follows by induction.
\end{proof}
\end{tcolorbox}
\end{condsection}
In the next subsections, we construct an optimal policy for $\psi$.
We first consider a simplified formula by replacing $\mathsf{r} \in \LTLSet$ with $\Vopt{\mathsf{r}} \in \mathrm{Prop}$.
We then construct a policy for $\psi$ by combining the prefix of the simplified policy with a suffix that handles the remaining $\mathsf{r}$.

\subsection{Policy for Until with AP by replacing $\mathsf{r}$ with $\Vopt{\mathsf{r}}$}
{
A challenge with $\psi$ is that $\mathsf{r} \in \LTLSet$ is arbitrary.
Thus, we first consider a simple but related formula $\varphi$.
For the same $q \in \mathrm{Prop}$ and $\mathsf{r} \in \LTLSet$,
}
consider $\varphi \coloneqq q \LTLUntil \Vopt{\mathsf{r}}$ and its finite-horizon counterpart $\tilde{\varphi} \coloneqq q \LTLUntil_{[0, \timer{}_0]} \Vopt{\mathsf{r}}$, which are defined similarly to $\psi$ and $\tilde{\psi}$.
We first show the optimal value is unchanged.
\begin{tcolorbox}[ThmFrame2]
\begin{lemma} \label{lem:replace_r_with_Vr}
   For any $n \geq 0$, and any $(x_{0:t}, \timer{}_0) \in \mathcal{X}^+ \times \mathbb{Z}_{\geq 0}$,
\begin{equation} \label{eq:replace_r_with_Vr}
      \Vhist{\tilde{\psi}}([x_{0:t}, \timer{}_0]) = \Vhist{\tilde{\varphi}}([x_{0:t}, \timer{}_0])
   \end{equation}
   In particular, this implies both $\Vhist{\psi}(x_{0:t}) = \Vhist{\varphi}(x_{0:t})$ and $\optwitness{\psi}(x_{0:t}) = \optwitness{\varphi}(x_{0:t})$, for all $x_{0:t} \in \mathcal{X}^+$.
\end{lemma}
\end{tcolorbox}
\begin{condproof}[lem:replace_r_with_Vr]
   \begin{align*}
      V^{*}_{[\psi_n]}(x)
      &= \smax{a_{0:\infty}} \smax{0 \leq t \leq n} \min\{ \mathsf{r}(x_{t:\infty}), \smin{0 \leq k < t} q(x_k) \}, \\
      &= \smax{0 \leq t \leq n} \smax{a_{0:t-1}} \min\{ \smax{a_{t:\infty}} \mathsf{r}(x_{t:\infty}), \smin{0 \leq k < t} q(x_k) \}, \\
      &= \smax{0 \leq t \leq n} \smax{a_{0:t-1}} \min\{ V^{*}_{[\mathsf{r}]}(x_t), \smin{0 \leq k < t} q(x_k) \}, \\
      &= \smax{a_{0:\infty}} \smax{0 \leq t \leq n} \min\{ V^{*}_{[\mathsf{r}]}(x_t), \smin{0 \leq k < t} q(x_k) \}, \\
      &= V^{*}_{[\varphi_n]}(x).
   \end{align*}
Finally, since the values are the same, the optimal witness times must also be the same.
\end{condproof}
We now construct a policy for $\varphi$ and show that it is optimal.
\begin{tcolorbox}[DefFrame2]
\begin{defi} \label{def:until_simple_optimal_policy}
Define the Markovian policy $\pi_{[\tilde{\varphi}]} : \tilde{\mathcal{X}} \to \mathcal{A}$ as
    \begin{equation} \label{eq:until_simple_policy_finite_horz}
        \pi_{[\tilde{\varphi}]}([x, \timer{}]) \coloneqq \argmax_{a} \Qopt{\tilde{\varphi}}([x, \timer{}], a).
    \end{equation}
    Now, define the non-Markovian policy $\pi_{[\varphi]} : \mathcal{X}^+ \to \mathcal{A}$ as
    \begin{equation} \label{eq:until_simple_policy}
\pi_{[\varphi]}(x_{0:k}) \coloneqq \pi_{[\tilde{\varphi}]}([x_k, \optwitness{\varphi}(x_{0:k}) - k]).
    \end{equation} 
\end{defi}
\end{tcolorbox}
\begin{tcolorbox}[ThmFrame2]
\begin{lemma} \label{lem:history_dp}
    For any trajectory $x_{0:\infty}$ and any $\timer{}_0$, let $s \in [0, \timer{}_0]$, and denote $c_\tau \coloneqq \min_{0 \leq k < \tau} q(x_k)$. Then,
    \begin{equation} \label{eq:history_dp_rho_psi}
    \begin{split}
    &\mathrel{\phantom{=}} \rho_{[\tilde{\psi}]}([x_{0:\infty}, \timer{}_0]) \\
    &= \smax{0 \leq \tau < s}
        \min\{ \mathsf{r}(x_{\tau:\infty}), c_\tau \}
        \lor
        \min\{ c_\tau, \rho_{[\tilde{\varphi}]}([x_{s:\infty}, \timer{}_{s}]) \},
    \end{split} \raisetag{5ex}
    \end{equation}
    and
    \begin{equation} \label{eq:history_dp_rho_varphi}
    \begin{split}
    &\mathrel{\phantom{=}} \rho_{[\tilde{\varphi}]}([x_{0:\infty}, \timer{}_0]) \\
    &= \smax{0 \leq \tau < s}
        \min\{ \Vopt{\mathsf{r}}(x_{\tau}), c_\tau \}
        \lor
        \min\{ c_\tau, \rho_{[\tilde{\varphi}]}([x_{s:\infty}, \timer{}_{s}]) \}.
    \end{split} \raisetag{5ex}
    \end{equation}
    Moreover, by maximizing over $a_{s:\infty}$, we get
    \begin{equation} \label{eq:history_dp_value}
    \begin{split}
        &\mathrel{\phantom{=}}
        \Vhist{\tilde{\psi}}([x_{0:s}, \timer{}_0]) =
        \Vhist{\tilde{\varphi}}([x_{0:s}, \timer{}_0]) \\
        &= \smax{0 \leq \tau < s} \min\{ \Vopt{\mathsf{r}}(x_{\tau}), c_\tau \} \lor \min\{ c_\tau, \Vopt{\tilde{\varphi}}([x_{s}, \timer{}_{s}]) \}.
    \end{split} \raisetag{5ex}
    \end{equation}
\end{lemma}
\end{tcolorbox}
A challenge to proving the optimality of $\pi_{[\varphi]}$ is that $\optwitness{\varphi}(x_{0:k}) - k$ may not follow the dynamics of the timer state $\timer{}$ in the augmented system for the optimality of $\pi_{[\tilde{\varphi}]}$ to transfer to $\pi_{[\varphi]}$.
We show next that this does not happen.
\begin{tcolorbox}[ThmFrame2]
\begin{lemma} \label{lem:witness_simulates_timer}
    Let $\bar{x}_{0:\infty}$ be generated by $\pi_{[\varphi]}$ from $\bar{x}_{0:s}$. Then,
    \begin{equation} \label{eq:witness_simulates_timer}
        \optwitness{\varphi}(\bar{x}_{0:k}) = \optwitness{\varphi}(\bar{x}_{0:s}), \quad \forall k \geq s.
    \end{equation}
\end{lemma}
\end{tcolorbox}
\begin{condproof}[lem:witness_simulates_timer]
Let $ \tau \coloneqq \optwitness{\varphi}(\bar{x}_{0:s})$.
We verify the assumptions of \Cref{lem:abstract_witness_persistence}
with $\chi=\varphi$ and $\pi=\pi_{[\varphi]}$.

Fix $k \ge s$, and assume
\[
\optwitness{\varphi}(\bar{x}_{0:k})=\tau.
\]
Let
\[
\bar{a}_k \coloneqq \pi_{[\varphi]}(\bar{x}_{0:k}).
\]
By \Cref{def:until_simple_optimal_policy},
\[
\bar{a}_k
=
\pi_{[\tilde{\varphi}]}([\bar{x}_k,\tau-k]).
\]
Since $\pi_{[\tilde{\varphi}]}$ is greedy with respect to
$\Qopt{\tilde{\varphi}}$, \Cref{lem:finite_infinite_horizon_Q} gives
\begin{align}
    \Qopt{\tilde{\varphi}}([\bar{x}_k,\tau-k],\bar{a}_k)
    &=
    \Vopt{\tilde{\varphi}}([\bar{x}_k,\tau-k]) \notag\\
    &=
    \Vhist{\varphi}(\bar{x}_{0:k}). \label{eq:witness_simulates_timer:new:tmp}
\end{align}

Since finite-horizon robustness is bounded above by infinite-horizon robustness
\eqref{eq:finite_horizon_ineq},
\[
\Qhist{\varphi}(\bar{x}_{0:k},\bar{a}_k)
\ge
\Qopt{\tilde{\varphi}}([\bar{x}_k,\tau-k],\bar{a}_k)
=
\Vhist{\varphi}(\bar{x}_{0:k})
\ge
\Qhist{\varphi}(\bar{x}_{0:k},\bar{a}_k).
\]
Hence
\[
\Qhist{\varphi}(\bar{x}_{0:k},\bar{a}_k)
=
\Vhist{\varphi}(\bar{x}_{0:k}).
\]
Applying \Cref{lem:Q_V_history},
\[
\Vhist{\varphi}(\bar{x}_{0:k+1})
=
\Qhist{\varphi}(\bar{x}_{0:k},\bar{a}_k)
=
\Vhist{\varphi}(\bar{x}_{0:k}),
\]
so assumption \Cref{item:cond_witness_persist_value} holds.

Next, by the definition of the finite-horizon Q-function and
\eqref{eq:witness_simulates_timer:new:tmp}, there exists a continuation
$\hat{x}_{0:\infty}$ of $\bar{x}_{0:k+1}$ such that
\[
\rho_{[\tilde{\varphi}]}([\hat{x}_{0:\infty},\tau])
=
\Vhist{\varphi}(\bar{x}_{0:k+1}).
\]
By \Cref{lem:finite_horizon_semantics}, the finite-horizon robustness on the left
is attained at some time $t \le \tau$. Since
\[
\rho_{[\tilde{\varphi}]}([\hat{x}_{0:\infty},\tau])
\le
\rho_{[\varphi]}(\hat{x}_{0:\infty})
\le
\Vhist{\varphi}(\bar{x}_{0:k+1}),
\]
we must in fact have
\[
\rho_{[\varphi]}(\hat{x}_{0:\infty})
=
\Vhist{\varphi}(\bar{x}_{0:k+1}),
\]
and the witness time of $\hat{x}_{0:\infty}$ for $\varphi$ is at most $\tau$.
Thus assumption \Cref{item:cond_witness_persist_attain} also holds.

Therefore, \Cref{lem:abstract_witness_persistence} implies
\[
\optwitness{\varphi}(\bar{x}_{0:k})
=
\optwitness{\varphi}(\bar{x}_{0:s}),
\qquad \forall k \ge s.
\]
\end{condproof}
\begin{condproof}[lem:witness_simulates_timer]
    \noindent\textbf{LHS $\geq$ RHS.}
    We first prove that $\optwitness{\varphi}(\bar{x}_{0:k}) \geq \optwitness{\varphi}(\bar{x}_{0:s})$.
    Suppose there exists $k > s$ such that $\optwitness{\varphi}(\bar{x}_{0:k}) < \optwitness{\varphi}(\bar{x}_{0:s})$.
    By definition of $\optwitness{\varphi}$, there exists a trajectory $\bar{x}_{0:\infty}$ extending $\bar{x}_{0:k}$ that achieves the optimal value $\Vhist{\varphi}(\bar{x}_{0:k})$ at time $\optwitness{\varphi}(\bar{x}_{0:k})$.
    Since $\optwitness{\varphi}(\bar{x}_{0:k}) < \optwitness{\varphi}(\bar{x}_{0:s})$, this trajectory also extends $\bar{x}_{0:s}$ and achieves the optimal value $\Vhist{\varphi}(\bar{x}_{0:s})$ at time $\optwitness{\varphi}(\bar{x}_{0:s})$.
    This contradicts the definition of $\optwitness{\varphi}(\bar{x}_{0:s})$ as the minimum witness time among all trajectories that achieve the optimal value $\Vhist{\varphi}(\bar{x}_{0:s})$.

    \noindent\textbf{LHS $\leq$ RHS.}
    Let $\bar{a}_k = \pi_{[\varphi]}(\bar{x}_{0:k})$ for $k \geq s$, and let $n_k \coloneqq \optwitness{\varphi}(\bar{x}_{0:k})$.
    Using \eqref{} from \Cref{lem:finite_infinite_horizon_Q},
    \begin{align*}
        \Qopt{\tilde{\varphi}}([\bar{x}_k, n_k - k], \bar{a}_k)
        &= \Vopt{\tilde{\varphi}}([\bar{x}_k, n_k - k]), \numberthis \label{eq:witness_simulates_timer:tmp} \\
        &= \Vhist{\tilde{\varphi}}([\bar{x}_{0:k}, n_k])
        = \Vhist{\varphi}(\bar{x}_{0:k}).
    \end{align*}
    Moreover, since the finite-horizon Q-value is bounded above by the infinite-horizon Q-value \eqref{eq:finite_horizon_ineq},
    \begin{align}
        \Qhist{\varphi}(\bar{x}_{0:k}, \bar{a}_k)
        &\geq \Qopt{\tilde{\varphi}}([\bar{x}_k, n_k - k], \bar{a}_k), \\
        &= \Vhist{\varphi}(\bar{x}_{0:k}) \geq \Qhist{\varphi}(\bar{x}_{0:k}, \bar{a}_k),
    \end{align}
    Applying \Cref{lem:Q_V_history} then gives
    \begin{equation}
        \Qhist{\varphi}(\bar{x}_{0:k}, \bar{a}_k) = \Vhist{\varphi}(\bar{x}_{0:k}) = \Vhist{\varphi}(\bar{x}_{0:k+1}).
    \end{equation}

    We now show the reverse inequality by induction. The base case of $k=s$ holds trivially.
    Assume the statement holds for some $k \geq s$.
    For $k+1$, since $Q_{[\varphi]}([\bar{x}_k, n_k - k], \bar{a}_k) = \Vhist{\varphi}(\bar{x}_{0:k})$ from \eqref{eq:witness_simulates_timer:tmp}, by definition of the optimal witness time $n_k \coloneqq \optwitness{\varphi}(\bar{x}_{0:k})$, there exists a continuation $\hat{x}_{0:\infty}$ of $\bar{x}_{0:k+1}$ such that
    \begin{equation}
        \rho_{[\tilde{\varphi}]}([\hat{x}_{0:\infty}, n_k])
        = \Vhist{\varphi}(\bar{x}_{0:k})
        = \Vhist{\varphi}(\bar{x}_{0:k+1}).
    \end{equation}
    By definition of the optimal witness time for $\bar{x}_{0:k+1}$,
    $\optwitness{\varphi}(\bar{x}_{0:k+1}) \leq n_k = \optwitness{\varphi}(\bar{x}_{0:k}) = \optwitness{\varphi}(\bar{x}_{0:s})$.
    Thus, \eqref{eq:witness_simulates_timer} holds by induction.
\end{condproof}
We are now ready to prove the optimality of $\pi_{[\varphi]}$.
\begin{tcolorbox}[ThmFrame2]
\begin{thm} \label{thm:simple_until_policy_optimal}
   $\pi_{[\varphi]}$ is optimal (\Cref{def:def_of_optimal_policy}) for TL formula $\varphi$.
\end{thm}
\end{tcolorbox}
\begin{tcolorbox}[AssmpFrame2]
\begin{proof}[Proof Sketch]
Fix $\bar{x}_{0:s}\in\mathcal X^+$, and let $\bar{x}_{0:\infty}$ be the
trajectory generated by $\pi_{[\varphi]}$ from $\bar{x}_{0:s}$.
Let
$\timer_s \coloneqq \optwitness{\varphi}(\bar{x}_{0:s})-s$, and
$\timer_0=\optwitness{\varphi}(\bar{x}_{0:s})$.
By \eqref{eq:finite_horizon_ineq} and \Cref{lem:finite_infinite_horizon_Q},
\[
\rho_{[\tilde{\varphi}]}([\bar{x}_{0:\infty},\timer_0])
\le
\rho_{[\varphi]}(\bar{x}_{0:\infty})
\le
\Vhist{\varphi}(\bar{x}_{0:s})
=
\Vhist{\tilde{\varphi}}([\bar{x}_{0:s},\timer_0]).
\]
It suffices to show the first and last term are equal.
By \Cref{lem:history_dp}, it is enough to prove
$
\rho_{[\tilde{\varphi}]}([\bar{x}_{s:\infty},\timer_s])
=
\Vopt{\tilde{\varphi}}([\bar{x}_s,\timer_s]).
$
By \Cref{lem:witness_simulates_timer}, the quantity $\optwitness{\varphi}(\bar{x}_{0:k})-k$ simulates the timer in the augmented system.
Hence the suffix follows the greedy finite-horizon policy
$\pi_{[\tilde{\varphi}]}$.

A backward induction on the timer therefore yields
$
\rho_{[\tilde{\varphi}]}([\bar{x}_{k:\infty},\timer_k])
=
\Vopt{\tilde{\varphi}}([\bar{x}_k,\timer_k])
$
for all $k\in[s,\optwitness{\varphi}(\bar{x}_{0:s})]$, and in particular at
$k=s$.
Therefore
$
\rho_{[\tilde{\varphi}]}([\bar{x}_{s:\infty},\timer_s])
=
\Vopt{\tilde{\varphi}}([\bar{x}_s,\timer_s])
$,
and the claim follows from the initial inequality chain.
\end{proof}
\end{tcolorbox}
\begin{condsection}[properties_of_varphi_tilde]
We now show properties of the optimal trajectory for $\tilde{\varphi}$ which we use to prove the optimality of a policy for $\psi$.
\begin{tcolorbox}[ThmFrame3]
\begin{lemma} \label{lem:until_Vr_finite_prefix}
    For fixed $x_{0:s} \in \mathcal{X}^+$,
    let $[x_{0:\infty}, \timer{}_0]$ be an optimal trajectory for $\tilde{\varphi}$, where $\timer{}_s = \optwitness{\varphi}(x_{0:s})$ and $\timer{}_0 = \timer{}_s + s$.
    Then, for any $\hat{x}_{0:\infty}$ that shares the first $\timer{}_0$ states,
$[\hat{x}_{0:\infty}, \timer{}_0]$ is also optimal for $\tilde{\varphi}$ given history $x_{0:s}$.
\end{lemma}
\end{tcolorbox}

\begin{tcolorbox}[ThmFrame3]
\begin{lemma} \label{lem:until_Vr_Vr_leq_rhs_varphi}
    Fix a history $x_{0:s}$, and let
    $n^* \coloneqq \optwitness{\varphi}(x_{0:s})$.
    Let $\bar{x}_{0:\infty}$ be a continuation of $x_{0:s}$ that achieves
    $\Vhist{\varphi}(x_{0:s})$ with witness time $n^*$.
    Then, for all $\tau \in [s, n^*-1]$,
    \begin{equation} \label{eq:until_Vr_Vr_leq_rhs_varphi}
        \Vopt{\mathsf{r}}(\bar{x}_\tau)
        \le
        q(\bar{x}_\tau)\land
        \smax{a}\Vopt{\tilde{\varphi}}([f(\bar{x}_\tau,a),\, n^*-(\tau+1)]).
    \end{equation}
\end{lemma}
\end{tcolorbox}
\end{condsection}

\subsection{A Markovian Policy for $\varphi$} \label{subsec:until_markovian_policy}
While $\pi_{[\varphi]}$ \eqref{eq:until_simple_policy} in \Cref{def:until_simple_optimal_policy} is optimal, evaluating the control at state $x_k$ requires
keeping track of the timer state in the augmented state space, resulting in a non-Markovian policy on the original state space $\mathcal{X}$.
Similar to \cite{sharpless2025dual}, we can construct an alternative \textit{Markovian} optimal policy $\hat{\pi}_{[\varphi]}$ that depends only on the current state $x_k$.
\begin{tcolorbox}[DefFrame2]
\begin{defi} \label{def:alternative_until_simple_optimal_policy}
Define the Markovian policy $\hat{\pi}_{[\varphi]}$ as
\begin{equation} \label{eq:alternative_until_simple_optimal_policy}
   \hat{\pi}_{[\varphi]}(x) \coloneqq \pi_{[\tilde{\varphi}]}([x, \optwitness{\varphi}(x)]),
\end{equation} 
where $\pi_{[\varphi]}$ is defined as in \eqref{eq:until_simple_policy_finite_horz}.
\end{defi}
\end{tcolorbox}
\begin{condsection}[alternate_until_simple]
Proving the optimality of $\hat{\pi}_{[\varphi]}$ is more involved than proving the optimality of $\pi_{[\varphi]}$ since it no longer has an explicit finite-horizon structure.
To prove the optimality of $\hat{\pi}_{[\varphi]}$, we will first prove a few useful lemmas.

\begin{tcolorbox}[ThmFrame2]
\begin{lemma} \label{lem:until_rho_nondecreasing}
    Let $\psi \coloneqq q \LTLUntil \mathsf{r}$.
    For any trajectory $x_{0:\infty}$, if $\optwitness{\psi}(x_{0:\infty}) > 0$, then 
    \begin{equation}
        \rho_{[\psi]}(x_{0:\infty}) \leq \rho_{[\psi]}(x_{1:\infty})
    \end{equation}
\end{lemma}
\end{tcolorbox}
\begin{proof}
    Since $\optwitness{\psi}(x_{0:\infty}) > 0$, we have
    \begin{align*}
        \rho_{[\psi]}(x_{0:\infty})
        &= \mathsf{r}(x_{0:\infty}) \lor (q(x_0) \land \rho_{[\psi]}(x_{1:\infty})), \\
        &= q(x_0) \land \rho_{[\psi]}(x_{1:\infty}), \\
        &\leq \rho_{[\psi]}(x_{1:\infty}). \qedhere
    \end{align*}
\end{proof}

\begin{tcolorbox}[ThmFrame2]
\begin{lemma} \label{lem:r0_is_witness}
    Let $\psi \coloneqq q \LTLUntil \mathsf{r}$.
    For any trajectory $x_{0:\infty}$,
    \begin{equation}
        \mathsf{r}(x_{0:\infty}) \geq q(x_0) \land \rho_{[\psi]}(x_{1:\infty}) \iff \witness{\psi}(x_{0:\infty}) = 0.
    \end{equation}
\end{lemma}
\end{tcolorbox}
\begin{proof}
    Assume the LHS holds. Then, using the recursive characterization of $\LTLUntil$,
    \begin{align}
        \rho_{[\psi]}(x_{0:\infty}) = \mathsf{r}(x_{0:\infty}) \lor (q(x_0) \land \rho_{[\psi]}(x_{1:\infty})) = \mathsf{r}(x_{0:\infty}).
    \end{align}
    This implies that $\optwitness{\psi}(x_{0:\infty}) = 0$.

    On the other hand, assume that $\optwitness{\psi}(x_{0:\infty}) = 0$. By definition of $\optwitness{\psi}$, $\rho_{[\psi]}(x_{0:\infty}) = \mathsf{r}(x_{0:\infty})$.
    By the recursive characterization of $\LTLUntil$, this implies that $\mathsf{r}(x_{0:\infty}) \geq q(x_0) \land \rho_{[\psi]}(x_{1:\infty})$.
\end{proof}

\begin{tcolorbox}[ThmFrame2]
\begin{lemma} \label{lem:finite_witness_exists}
    For any trajectory $x_{0:\infty}$, there exists some finite time $\tau \geq 0$ such that
    \begin{equation}
        \witness{\psi}(x_{\tau:\infty}) = 0.
    \end{equation}
\end{lemma}
\end{tcolorbox}

\begin{proof}
    For convenience, define $\rho_\tau \coloneqq \rho_{[\psi]}(x_{\tau:\infty})$ for $\tau \geq 0$. 
    The claim holds trivially if $\rho_0 = 0$.
    It remains to show this holds for $\rho_0 > 0$.

    We prove this by contradiction.
    Assume that $t_\tau > 0$ for all $\tau \geq 0$. 
    Then, by \Cref{lem:until_rho_nondecreasing}, $\rho_\tau$ is a nondecreasing sequence. By \Cref{assmp:max_attainable:finite} or \Cref{assmp:max_attainable:compact}, $\rho_\tau$ takes values in a finite set.
    Hence, $\rho_\tau$ converges to some value $M$ in finite time, i.e., there exists $K \geq 0$ and $M \in \mathbb{R}$ such that
    \begin{equation} \label{eq:proof:witness_zero:rho_M}
        \rho_\tau = M,\quad \forall \tau \geq K.
    \end{equation}
    Fix such a $K$. Since $t_K$ is a witness time, by \eqref{eq:proof:witness_zero:rho_M} and \eqref{eq:witness_exist:1},
    \begin{equation}
        M = \rho_K = \min\{ \mathsf{r}(x_{K + t_K:\infty}), \smin{0 \leq k < t_K} q(x_{K+k}) \}.
    \end{equation}
    In particular, $\mathsf{r}(x_{K+t_K:\infty}) \geq M$.
    Now, using \eqref{eq:proof:witness_zero:rho_M} again,
    \begin{equation}
        M = \rho_{K + t_K} \geq \mathsf{r}(x_{K + t_K:\infty}) \geq M.
    \end{equation}
    Hence, $\rho_{K + t_K} = \mathsf{r}(x_{K + t_K:\infty}) = M$.
    However, by \Cref{lem:r0_is_witness}, this implies that $t_{K + t_K}=0$, which contradicts the assumption, so the claim holds.
\end{proof}

We are now ready to prove the optimality of $\hat{\pi}$.
\end{condsection}
\begin{tcolorbox}[ThmFrame3]
\begin{thm} \label{thm:alternative_until_simple_optimal_policy_optimal}
   $\hat{\pi}_{[\varphi]}$ is also optimal for $\varphi$.
\end{thm}
\end{tcolorbox}

\begin{tcolorbox}[NoteFrame2]
\begin{rmk} \label{rmk:until_finite_horizon_value_iteration}
    By \Cref{lem:finite_infinite_horizon_Q}, the infinite-horizon value function can be solved by taking a limit of the values of finite-horizon problems.
    Consequently, this process yields the value functions and optimal actions for all states for all horizon lengths.
    $\hat{\pi}_{[\varphi]}$ can be computed during this process by storing the action that maximizes the Q function for the first step $n$ where $\Vopt{\tilde{\varphi}}([x, n])$ converges.
\end{rmk}
\end{tcolorbox}

\begin{condproof}[thm:alternative_until_simple_optimal_policy_optimal]
Fix a $\bar{x}_0 \in \mathcal{X}$,
and let $\bar{x}_{0:\infty}$ and $\bar{a}_{0:\infty}$ be the state and action trajectories generated by $\hat{\pi}_{[\varphi]}$ from $\bar{x}_0$.
By \Cref{lem:finite_witness_exists}, for every $s \geq 0$, there exists a $n^*_s \coloneqq \tau(\bar{x}_s) \geq 0$ such that $\witness{\varphi}(\bar{x}_{s + n^*_s}) = 0$.
We prove
\begin{equation} \label{eq:proof:alternative_until_simple_optimal_policy:induction}
    \rho_{[\varphi]}(\bar{x}_{0:\infty}) = \Vopt{\varphi}(\bar{x}_0)      
\end{equation}
via induction on $\tau(\bar{x}_0)$.

\noindent\textbf{Base case ($\tau(\bar{x}_0) = 0$)}: Since $n^*_0 = 0$,
\begin{equation}
    \rho_{[\varphi]}(\bar{x}_{0:\infty})
    = \Vopt{\mathsf{r}}(\bar{x}_0)
    = \Vopt{\varphi_0}(\bar{x}_0)
    = \Vopt{\varphi}(\bar{x}_0).
\end{equation}

\noindent\textbf{Induction step}: Assume \eqref{eq:proof:alternative_until_simple_optimal_policy:induction} holds for all $x_0$ where $\tau(x_0) < \tau(\bar{x}_0)$ and $\tau(\bar{x}_0) > 0$.
Since $\hat{\pi}_{[\varphi]}$ \eqref{eq:alternative_until_simple_optimal_policy} is Markovian,
the suffix trajectory $\bar{x}_{1:\infty}$ is generated by $\hat{\pi}_{[\varphi]}$ from $\bar{x}_1$.
Thus, $\tau(\bar{x}_1) = \tau(\bar{x}_0) - 1$. 
By the recursive characterization of $\LTLUntil$ and the induction hypothesis \eqref{eq:proof:alternative_until_simple_optimal_policy:induction},
\begin{align*}
    \rho_{[\varphi]}(\bar{x}_{0:\infty})
    &= \Vopt{\mathsf{r}}(\bar{x}_0) \lor \big( q(\bar{x}_0) \land \rho_{[\varphi]}(\bar{x}_{1:\infty}) \big), \\
    &= \Vopt{\mathsf{r}}(\bar{x}_0) \lor \big( q(\bar{x}_0) \land \Vopt{\varphi}(\bar{x}_1) \big). \numberthis \label{eq:proof:alternative_until_simple_optimal_policy:rho_varphi_recursive}
\end{align*}
We now prove \eqref{eq:proof:alternative_until_simple_optimal_policy:induction} holds for $\bar{x}_0$. 
The upper bound $\rho_{[\varphi]}(\bar{x}_{0:\infty}) \leq \Vopt{\varphi}(\bar{x}_0)$ is immediate by the definition of $\Vopt{\varphi}$. It remains to show the lower bound.

By \Cref{lem:until_Vr_Vr_leq_rhs_varphi}, since $n^*_0 > 0$,
\begin{equation}
    V_{[\mathsf{r}]}(\bar{x}_0) \leq \smax{a} q(\bar{x}_0) \land \Vopt{\varphi_{n^*_0 - 1}}(f(\bar{x}_0, a)).
\end{equation}
Thus, the action $\bar{a}_0$ satisfies
\begin{align}
    \bar{a}_0 \in \argmax_a q(\bar{x}_0) \land \Vopt{\varphi_{n^*_0 - 1}}(f(\bar{x}_0, a)),
\end{align}
and
\begin{equation}
    \Vopt{\varphi}(\bar{x}_0) = q(\bar{x}_0) \land \Vopt{\varphi_{n^*_0 - 1}}(\bar{x}_1).
\end{equation}
Since the infinite-horizon value is at least the finite-horizon value \eqref{eq:finite_horizon_ineq}, $\Vopt{\varphi}(\bar{x}_1) \geq \Vopt{\varphi_{n^*_0 - 1}}(\bar{x}_1)$, and we get
\begin{equation}
    q(\bar{x}_0) \land \Vopt{\varphi}(\bar{x}_1) \geq q(\bar{x}_0) \land \Vopt{\varphi_{n^*_0 - 1}}(\bar{x}_1) = \Vopt{\varphi}(\bar{x}_0).
\end{equation}
Returning to \eqref{eq:proof:alternative_until_simple_optimal_policy:rho_varphi_recursive},
\begin{equation}
    \rho_{[\varphi]}(\bar{x}_{0:\infty})
    \geq q(\bar{x}_0) \land \Vopt{\varphi}(\bar{x}_1)
    \geq \Vopt{\varphi}(\bar{x}_0).
\end{equation}
Combining both bounds, $\rho_{[\varphi]}(\bar{x}_{0:\infty}) = \Vopt{\varphi}(\bar{x}_0)$. The claim thus holds by induction.
\end{condproof}

\subsection{Policy for Until with General $\mathsf{r}$} \label{sec:until:general_r}
We now tackle the original problem of finding an optimal policy for general $\mathsf{r} \in \LTLSet$.
We propose two candidate policies corresponding to the two optimal policies for $\varphi$ defined in \Cref{def:until_simple_optimal_policy} and \Cref{def:alternative_until_simple_optimal_policy}, respectively.
For both, we make the following assumption, which we later show is satisfied for a large class of TL formula
(\Cref{sec:composition}).
\begin{tcolorbox}[AssmpFrame2]
\begin{assmp} \label{assmp:exists_optimal_policy_for_r}
    There exists a \textcolor{black}{known (potentially non-Markovian)} optimal policy $\pi_{[\mathsf{r}]}$ for $\mathsf{r}$,
    i.e., for $x_{0:\infty}$ generated by $\pi_{[\mathsf{r}]}$ from $x_{0:s}$,
    \begin{equation} \label{eq:pi_r_optimal}
        \rho_{[\mathsf{r}]}(x_{0:\infty}) = \Vopt{\mathsf{r}}(x_{0:s}).
    \end{equation}
\end{assmp}
\end{tcolorbox}
Note that \Cref{assmp:exists_optimal_policy_for_r} holds for $\mathsf{r} \in \mathrm{Prop}$ (\Cref{def:S:Prop}).

\begin{tcolorbox}[DefFrame2]
\begin{defi} \label{def:until_optimal_policy}
    Define the policy $\pi_{[\psi]}$ as
    \begin{equation} \label{eq:until_optimal_policy}
        \pi_{[\psi]}(x_{0:k}) \coloneqq \begin{dcases}
            \pi_{[\varphi]}(x_{0:k}), & n_k - k > 0, \\
            \pi_{[\mathsf{r}]}(x_{n_k:k}), & n_k - k \leq 0,
        \end{dcases}
    \end{equation} 
    where $n_k \coloneqq \optwitness{\psi}(x_{0:k})$
    and $\pi_{[\varphi]}$ is defined in \eqref{eq:until_simple_policy}.
\end{defi}
\end{tcolorbox}
\begin{tcolorbox}[DefFrame2]
\begin{defi} \label{def:alternative_until_optimal_policy}
    Define the policy $\hat{\pi}_{[\psi]}$ as
    \begin{equation} \label{eq:alternative_until_optimal_policy}
        \hat{\pi}_{[\psi]}(x_{0:k}) \coloneqq \begin{dcases}
            \hat{\pi}_{[\varphi]}(x_k), & k < n^*, \\
            \pi_{[\mathsf{r}]}(x_{n^*:k}), & k \geq n^*,
        \end{dcases}
    \end{equation} 
where $n^* = \min\{ n : \optwitness{\varphi}(x_n) = 0 \}$
    is the smallest time such that $\optwitness{\varphi}(x_{n^*}) = 0$, and $\hat{\pi}_{[\varphi]}$ is defined in \eqref{eq:alternative_until_simple_optimal_policy}.
\end{defi}
\end{tcolorbox}
\begin{tcolorbox}[NoteFrame2]
\begin{rmk}
    $\hat{\pi}_{[\psi]}$ is Markovian on the original state space $\mathcal{X}$ if $\pi_{[\mathsf{r}]}$ is Markovian on $\mathcal{X}$.
\end{rmk}
\end{tcolorbox}

\begin{condsection}[lem:witness_fixed_after_switch]
\begin{tcolorbox}[ThmFrame3]
\let\label\origlabel
\begin{lemma} \label{lem:witness_fixed_after_switch}
Fix a history $x_{0:s}\in\mathcal X^+$ and let
$\tau \coloneqq \optwitness{\psi}(x_{0:s})$.
Assume $\tau<s$.
Let $\bar x_{0:\infty}$ be the trajectory generated by $\pi_{[\psi]}$ from
$x_{0:s}$, and define $c_\tau \coloneqq \smin{0\le i<\tau} q(\bar x_i)$.
Then:
\begin{enumerate}[label=(\roman*), ref=\thelemma-(\roman*)]
    \item \label[lemma]{item:witness_fixed_after_switch:witness}
    For all $k \geq s$, $\optwitness{\psi}(\bar x_{0:k})=\tau$.
\item \label[lemma]{item:witness_fixed_after_switch:policy}
    For all $k \geq s$, $\pi_{[\psi]}(\bar x_{0:k})=\pi_{[\mathsf r]}(\bar x_{\tau:k})$.
\item \label[lemma]{item:witness_fixed_after_switch:suffix}
    the suffix $\bar x_{\tau:\infty}$ is exactly the trajectory generated by
    $\pi_{[\mathsf r]}$ from $\bar x_{\tau:s}$, and therefore
    $
    \rho_{[\mathsf r]}(\bar x_{\tau:\infty})
    =
    \Vhist{\mathsf r}(\bar x_{\tau:s})
    $.
\item \label[lemma]{item:witness_fixed_after_switch:rho}
    $\rho_{[\psi]}(\bar x_{0:\infty}) \ge \min\{c_\tau,\Vhist{\mathsf r}(\bar x_{\tau:s})\}$.
\end{enumerate}
\end{lemma}
\end{tcolorbox}
\end{condsection}
\begin{condproof}[lem:witness_fixed_after_switch]
We first prove \Cref{item:witness_fixed_after_switch:witness}.
We verify the assumptions of \Cref{lem:abstract_witness_persistence}
with $\chi=\psi$, $\pi=\pi_{[\psi]}$, initial history $x_{0:s}$, and
$\tau=\optwitness{\psi}(x_{0:s})$.
Let $\bar{x}_{0:\infty}$ be the trajectory generated by $\pi_{[\psi]}$
from $x_{0:s}$.
Fix $k \ge s$, and assume
$\optwitness{\psi}(\bar{x}_{0:k})=\tau$.
Let
$\bar{a}_k \coloneqq \pi_{[\psi]}(\bar{x}_{0:k})$.
Since $\tau< s \le k$, we have $\tau-k<0$, so the definition of $\pi_{[\psi]}$
gives
\begin{equation} \label{eq:witness_fixed_after_switch:use_pi_r}
    \bar{a}_k
    =
    \pi_{[\mathsf r]}(\bar{x}_{\tau:k}).
\end{equation}
We first show that
\begin{equation} \label{eq:witness_fixed_after_switch:r_value_const}
    \Vhist{\mathsf r}(\bar{x}_{\tau:k+1})
    =
    \Vhist{\mathsf r}(\bar{x}_{\tau:k}).
\end{equation}
Since $\pi_{[\mathsf r]}$ is optimal over histories for $\mathsf r$,
the trajectory generated by $\pi_{[\mathsf r]}$ from $\bar{x}_{\tau:k}$
achieves $\Vhist{\mathsf r}(\bar{x}_{\tau:k})$.
Its first action is precisely $\bar a_k$ by
\eqref{eq:witness_fixed_after_switch:use_pi_r}, so
\[
\Qhist{\mathsf r}(\bar{x}_{\tau:k},\bar a_k)
=
\Vhist{\mathsf r}(\bar{x}_{\tau:k}).
\]
Applying \Cref{lem:Q_V_history} to $\mathsf r$ yields
\[
\Vhist{\mathsf r}(\bar{x}_{\tau:k+1})
=
\Qhist{\mathsf r}(\bar{x}_{\tau:k},\bar a_k)
=
\Vhist{\mathsf r}(\bar{x}_{\tau:k}),
\]
proving \eqref{eq:witness_fixed_after_switch:r_value_const}.

Next, we show that
\begin{equation} \label{eq:witness_fixed_after_switch:psi_value_repr}
    \Vhist{\psi}(\bar{x}_{0:k})
    =
    \min\{c_\tau,\Vhist{\mathsf r}(\bar{x}_{\tau:k})\}.
\end{equation}
Since $\optwitness{\psi}(\bar{x}_{0:k})=\tau$, there exists a continuation
$y_{0:\infty}$ of $\bar{x}_{0:k}$ such that
\[
\rho_{[\psi]}(y_{0:\infty})=\Vhist{\psi}(\bar{x}_{0:k})
\]
and
\[
\witness{\psi}(y_{0:\infty})=\tau.
\]
Because the witness time is $\tau$, and the prefix up to time $\tau$ agrees with
$\bar x_{0:\tau}$,
\[
\Vhist{\psi}(\bar{x}_{0:k})
=
\rho_{[\psi]}(y_{0:\infty})
=
\min\{c_\tau,\rho_{[\mathsf r]}(y_{\tau:\infty})\}
\le
\min\{c_\tau,\Vhist{\mathsf r}(\bar{x}_{\tau:k})\}.
\]
Conversely, let $\hat{x}_{\tau:\infty}$ be the trajectory generated by
$\pi_{[\mathsf r]}$ from $\bar{x}_{\tau:k}$.
By the history-optimality of $\pi_{[\mathsf r]}$,
\[
\rho_{[\mathsf r]}(\hat{x}_{\tau:\infty})
=
\Vhist{\mathsf r}(\bar{x}_{\tau:k}).
\]
Concatenating the fixed prefix $\bar{x}_{0:\tau}$ with this suffix gives a
continuation of $\bar{x}_{0:k}$ whose witness time for $\psi$ is $\tau$ and whose
robustness is
\[
\min\{c_\tau,\Vhist{\mathsf r}(\bar{x}_{\tau:k})\}.
\]
By the definition of $\Vhist{\psi}(\bar{x}_{0:k})$, this implies
\[
\Vhist{\psi}(\bar{x}_{0:k})
\ge
\min\{c_\tau,\Vhist{\mathsf r}(\bar{x}_{\tau:k})\}.
\]
Thus \eqref{eq:witness_fixed_after_switch:psi_value_repr} holds.

Using \eqref{eq:witness_fixed_after_switch:psi_value_repr} and
\eqref{eq:witness_fixed_after_switch:r_value_const},
\[
\Vhist{\psi}(\bar{x}_{0:k})
=
\min\{c_\tau,\Vhist{\mathsf r}(\bar{x}_{\tau:k+1})\}.
\]
Now let $\tilde{x}_{\tau:\infty}$ be the trajectory generated by
$\pi_{[\mathsf r]}$ from $\bar{x}_{\tau:k+1}$.
Again by history-optimality of $\pi_{[\mathsf r]}$,
\[
\rho_{[\mathsf r]}(\tilde{x}_{\tau:\infty})
=
\Vhist{\mathsf r}(\bar{x}_{\tau:k+1}).
\]
Concatenating the fixed prefix $\bar{x}_{0:\tau}$ with this suffix yields a
continuation of $\bar{x}_{0:k+1}$ whose witness time is exactly $\tau$ and whose
robustness is
\[
\min\{c_\tau,\Vhist{\mathsf r}(\bar{x}_{\tau:k+1})\}
=
\Vhist{\psi}(\bar{x}_{0:k}).
\]
Hence
\[
\Vhist{\psi}(\bar{x}_{0:k+1})
\ge
\Vhist{\psi}(\bar{x}_{0:k}).
\]
On the other hand, every continuation of $\bar{x}_{0:k+1}$ is also a continuation
of $\bar{x}_{0:k}$, so
\[
\Vhist{\psi}(\bar{x}_{0:k+1})
\le
\Vhist{\psi}(\bar{x}_{0:k}).
\]
Therefore,
\[
\Vhist{\psi}(\bar{x}_{0:k+1})
=
\Vhist{\psi}(\bar{x}_{0:k}),
\]
so assumption \Cref{item:cond_witness_persist_value} holds.

Moreover, the continuation just constructed from $\bar{x}_{0:k+1}$ achieves
$\Vhist{\psi}(\bar{x}_{0:k+1})$ with witness time exactly $\tau$, and therefore
with witness time at most $\tau$.
Thus assumption \Cref{item:cond_witness_persist_attain} also holds.

Applying \Cref{lem:abstract_witness_persistence} gives
\[
\optwitness{\psi}(\bar{x}_{0:k})=\tau,
\qquad \forall k \ge s,
\]
which proves \Cref{item:witness_fixed_after_switch:witness}.

We next prove \Cref{item:witness_fixed_after_switch:policy}.
Fix any $k\ge s$. By \Cref{item:witness_fixed_after_switch:witness},
\[
\optwitness{\psi}(\bar x_{0:k})=\tau.
\]
Since $\tau-k<0$, the definition of $\pi_{[\psi]}$ yields
\[
\pi_{[\psi]}(\bar x_{0:k})
=
\pi_{[\mathsf r]}(\bar x_{\tau:k}),
\]
which is exactly \Cref{item:witness_fixed_after_switch:policy}.

We now prove \Cref{item:witness_fixed_after_switch:suffix}.
The suffix $\bar x_{\tau:s}$ is fixed.
For every $k\ge s$, the action applied by $\pi_{[\psi]}$ at time $k$ equals
$\pi_{[\mathsf r]}(\bar x_{\tau:k})$ by
\Cref{item:witness_fixed_after_switch:policy}.
Hence, from time $s$ onward, the evolution of the suffix
$\bar x_{\tau:\infty}$ follows exactly the same dynamics and the same actions as
the trajectory generated by $\pi_{[\mathsf r]}$ from $\bar x_{\tau:s}$.
Therefore $\bar x_{\tau:\infty}$ is exactly that trajectory.
Since $\pi_{[\mathsf r]}$ is optimal over histories for $\mathsf r$,
\[
\rho_{[\mathsf r]}(\bar x_{\tau:\infty})
=
\Vhist{\mathsf r}(\bar x_{\tau:s}).
\]

Finally, we prove \Cref{item:witness_fixed_after_switch:rho}.
By the semantics of Until,
\[
\rho_{[\psi]}(\bar x_{0:\infty})
=
\smax{t\ge 0}
\min\!\left\{
\rho_{[\mathsf r]}(\bar x_{t:\infty}),
\smin{0\le i<t} q(\bar x_i)
\right\}.
\]
Evaluating the supremum at $t=\tau$ gives
\[
\rho_{[\psi]}(\bar x_{0:\infty})
\ge
\min\{c_\tau,\rho_{[\mathsf r]}(\bar x_{\tau:\infty})\}.
\]
Using \eqref{item:witness_fixed_after_switch:suffix},
\[
\rho_{[\psi]}(\bar x_{0:\infty})
\ge
\min\{c_\tau,\Vhist{\mathsf r}(\bar x_{\tau:s})\},
\]
as claimed.
\end{condproof}

\begin{condsection}[lem:past_witness_value_identity]
\begin{tcolorbox}[ThmFrame2]
\begin{lemma}[Value identity after the switch] \label{lem:past_witness_value_identity}
Fix a history $x_{0:s}\in\mathcal X^+$ and let
$\tau \coloneqq \optwitness{\psi}(x_{0:s})$.
Assume $\tau < s$, and define $c_\tau \coloneqq \smin{0\le i<\tau} q(x_i)$.
Then,
\begin{equation}
\Vhist{\psi}(x_{0:s})
= \min\{c_\tau,\Vhist{\mathsf r}(x_{\tau:s})\}.
\end{equation}
\end{lemma}
\end{tcolorbox}
\end{condsection}
\begin{condproof}[lem:past_witness_value_identity]
We first show
\[
\Vhist{\psi}(x_{0:s}) \le \min\{c_\tau,\Vhist{\mathsf r}(x_{\tau:s})\}.
\]
Since $\tau=\optwitness{\psi}(x_{0:s})$, by definition of optimal witness time
there exists a continuation $\bar x_{0:\infty}$ of $x_{0:s}$ such that
\[
\rho_{[\psi]}(\bar x_{0:\infty})=\Vhist{\psi}(x_{0:s}),
\quad\text{and}\quad
\witness{\psi}(\bar x_{0:\infty})=\tau.
\]
Because the witness time is $\tau$,
\[
\Vhist{\psi}(x_{0:s})
=
\min\{c_\tau,\rho_{[\mathsf r]}(\bar x_{\tau:\infty})\}.
\]
Now $\bar x_{\tau:\infty}$ is a continuation of the history $x_{\tau:s}$, so by
the definition of $\Vhist{\mathsf r}$,
\[
\rho_{[\mathsf r]}(\bar x_{\tau:\infty})
\le
\Vhist{\mathsf r}(x_{\tau:s}).
\]
Therefore
\[
\Vhist{\psi}(x_{0:s})
=
\min\{c_\tau,\rho_{[\mathsf r]}(\bar x_{\tau:\infty})\}
\le
\min\{c_\tau,\Vhist{\mathsf r}(x_{\tau:s})\}.
\]
For the reverse inequality, let $\hat x_{\tau:\infty}$ be a continuation of
$x_{\tau:s}$ that attains $\Vhist{\mathsf r}(x_{\tau:s})$, i.e.,
$ \rho_{[\mathsf r]}(\hat x_{\tau:\infty})
= \Vhist{\mathsf r}(x_{\tau:s})$.
Concatenate the fixed prefix $x_{0:\tau}$ with this continuation, i.e., $\hat{x}_{0:\infty}$.
This gives a
continuation of $x_{0:s}$ for which the candidate witness time $\tau$ yielding
robustness
\[
\rho_{[\psi]}(\hat x_{0:\infty})
= \min\{c_\tau,\rho_{[\mathsf r]}(\hat x_{\tau:\infty})\}
= \min\{c_\tau,\Vhist{\mathsf r}(x_{\tau:s})\}.
\]
Hence, by the definition of $\Vhist{\psi}$,
\[
\Vhist{\psi}(x_{0:s})
\ge
\rho_{[\psi]}(\hat x_{0:\infty})
=
\min\{c_\tau,\Vhist{\mathsf r}(x_{\tau:s})\}.
\]

Combining the two inequalities gives
\[
\Vhist{\psi}(x_{0:s})
=
\min\{c_\tau,\Vhist{\mathsf r}(x_{\tau:s})\}.
\]
\end{condproof}

\begin{condsection}
\begin{tcolorbox}[ThmFrame3]
\begin{lemma} \label{lem:pi_psi_matches_pi_varphi_before_switch}
    For a fixed $x_{0:s}$,
    let $\hat{x}_{0:\infty}$ and $\bar{x}_{0:\infty}$ be generated from $x_{0:s}$ by $\pi_{[\varphi]}$ \eqref{eq:until_simple_policy} and $\pi_{[\psi]}$ \eqref{eq:until_optimal_policy} respectively.
    Define $n^* \coloneqq \optwitness{\psi}(x_{0:s})$. Then,
    \begin{equation} \label{eq:pi_psi_matches_pi_varphi_before_switch}
        \bar{x}_{0:n^*} = \hat{x}_{0:n^*}  
    \end{equation}
\end{lemma}
\end{tcolorbox}
\begin{tcolorbox}[ProofFrame3]
\begin{proof}
We prove this by induction on $k \in [s,n^*-1]$.
The claim is trivial for $k=s$.
Assume $\bar{x}_{0:k}=\hat{x}_{0:k}$ for some $k<n^*$.
Since $\hat{x}$ is generated by $\pi_{[\varphi]}$, \Cref{lem:witness_simulates_timer}
gives
\[
\optwitness{\varphi}(\hat{x}_{0:k})=\optwitness{\varphi}(\hat{x}_{0:s})=n^*.
\]
Using \Cref{lem:replace_r_with_Vr} and the inductive hypothesis,
\[
\optwitness{\psi}(\bar{x}_{0:k})
=
\optwitness{\varphi}(\bar{x}_{0:k})
=
\optwitness{\varphi}(\hat{x}_{0:k})
=
n^*.
\]
Hence $n^*-k>0$, so by the definition of $\pi_{[\psi]}$,
\[
\pi_{[\psi]}(\bar{x}_{0:k})=\pi_{[\varphi]}(\bar{x}_{0:k})
=\pi_{[\varphi]}(\hat{x}_{0:k}).
\]
Therefore the next states also agree:
\[
\bar{x}_{k+1}=f(\bar{x}_k,\pi_{[\psi]}(\bar{x}_{0:k}))
           =f(\hat{x}_k,\pi_{[\varphi]}(\hat{x}_{0:k}))
           =\hat{x}_{k+1}.
\]
The claim thus holds by induction.
\end{proof}
\end{tcolorbox}
\end{condsection}
\begin{condsection}[lem:late_optimal_witness_suffix_value]
\begin{tcolorbox}[ThmFrame3]
\begin{lemma}[Late optimal witness implies suffix value representation]
\label{lem:late_optimal_witness_suffix_value}
Fix a history $x_{0:s}\in\mathcal X^+$ and let
$
n^* \coloneqq \optwitness{\varphi}(x_{0:s})$ and $\timer{}_s \coloneqq n^*-s
$.
Assume $n^*\ge s$, and define
$
c_s \coloneqq \smin{0\le i<s} q(x_i)
$.
Then
\begin{equation} \label{eq:late_optimal_witness_suffix_value}
    \Vhist{\varphi}(x_{0:s})
    =
    \min\{c_s,\Vopt{\tilde{\varphi}}([x_s,\timer{}_s])\}.
\end{equation}
Moreover, if $\bar x_{0:\infty}$ is any continuation of $x_{0:s}$ such that
$[\bar x_{s:\infty},\timer{}_s]$ is optimal for $\tilde\varphi$, then
\[
\witness{\tilde\varphi}([\bar x_{s:\infty},\timer{}_s])=\timer{}_s.
\]
Equivalently,
\begin{equation}
\rho_{[\tilde{\varphi}]}([\bar{x}_{s:\infty},\timer{}_s])
=
\min\!\left\{
\Vopt{\mathsf r}(\bar{x}_{n^*}),
\,
\smin{s\le i<n^*} q(\bar{x}_i)
\right\}.
\end{equation}
\end{lemma}
\end{tcolorbox}
\end{condsection}
\begin{condproof}[lem:late_optimal_witness_suffix_value]
Let $ w \coloneqq \witness{\tilde{\varphi}}([\bar{x}_{s:\infty},\timer{}_s])$.
By the finite-horizon semantics of $\tilde{\varphi}$, any witness time lies in
$[0,\timer{}_s]$, so
$ w \le \timer{}_s$.
Let $c_\tau \coloneqq \smin{0\le i<\tau} q(\bar{x}_i)$ for $\tau \ge 0$.
Consider the first term in \eqref{eq:history_dp_value} which we call $A$, i.e.,
\begin{equation}
A \coloneqq \smax{0\le \tau < s}\min\{\Vopt{\mathsf r}(\bar{x}_\tau),c_\tau\} \leq \Vhist{\varphi}(x_{0:s}).
\end{equation}
We claim that the inequality is strict, i.e.,
\begin{equation} \label{eq:tilde_varphi_terminal_witness:A_strict}
A < \Vhist{\varphi}(x_{0:s}).
\end{equation}
We prove this by contradiction. Suppose there exists $\tau^* \in [0, s-1]$ such that
$\min\{\Vopt{\mathsf r}(\bar{x}_{\tau^*}),c_{\tau^*}\}
= \Vhist{\varphi}(x_{0:s})$.
This implies $\tau^* < s \leq n^*$ is a witness time for $\varphi$ from $x_{0:s}$, contradicting the minimality of $n^*=\optwitness{\varphi}(x_{0:s})$. Thus, \eqref{eq:tilde_varphi_terminal_witness:A_strict} holds.

Next, by \Cref{lem:finite_infinite_horizon_Q,lem:replace_r_with_Vr},
\begin{equation} \label{eq:tilde_varphi_terminal_witness:history_eq}
    \Vhist{\varphi}(x_{0:s})
    =
    \Vhist{\tilde{\varphi}}([x_{0:s},n^*]).
\end{equation}
Applying \eqref{eq:history_dp_value} with $\timer{}_0=n^*$ and using
\eqref{eq:tilde_varphi_terminal_witness:A_strict}, we obtain
\begin{align}
    \Vhist{\varphi}(x_{0:s})
    &= A \lor \min\{c_s,\Vopt{\tilde{\varphi}}([\bar{x}_s,\timer{}_s])\} \notag\\
    &= \min\{c_s,\Vopt{\tilde{\varphi}}([\bar{x}_s,\timer{}_s])\}. \label{eq:tilde_varphi_terminal_witness:history_suffix}
\end{align}
Since $[\bar{x}_{s:\infty},\timer{}_s]$ is optimal for $\tilde{\varphi}$,
\eqref{eq:tilde_varphi_terminal_witness:history_suffix} becomes
\begin{equation} \label{eq:tilde_varphi_terminal_witness:history_suffix_rho}
    \Vhist{\varphi}(x_{0:s})
    =
    \min\{c_s,\rho_{[\tilde{\varphi}]}([\bar{x}_{s:\infty},\timer{}_s])\}.
\end{equation}
Since $w$ is the witness time of
$[\bar{x}_{s:\infty},\timer{}_s]$ for $\tilde{\varphi}$,
\[
\rho_{[\tilde{\varphi}]}([\bar{x}_{s:\infty},\timer{}_s])
=
\min\!\left\{
\Vopt{\mathsf r}(\bar{x}_{s+w}),
\,
\smin{s\le i<s+w} q(\bar{x}_i)
\right\}.
\]
Substituting this into
\eqref{eq:tilde_varphi_terminal_witness:history_suffix_rho} gives
\[
\Vhist{\varphi}(x_{0:s})
=
\min\!\left\{
\Vopt{\mathsf r}(\bar{x}_{s+w}),
\,
\smin{0\le i<s+w} q(\bar{x}_i)
\right\}.
\]
Thus the continuation $\bar{x}_{0:\infty}$ attains the optimal value
$\Vhist{\varphi}(x_{0:s})$ with witness time $s+w$.
Suppose $w<\timer{}_s=n^*-s$, so that $s+w<n^*$.
This contradicts the definition of $n^*=\optwitness{\varphi}(x_{0:s})$, and $w \geq \timer{}_s$.
Having shown the reverse direction already, we must have $w = \timer{}_s$.

The equivalent representation of
$\rho_{[\tilde{\varphi}]}([\bar{x}_{s:\infty},\timer{}_s])$
follows immediately from the finite-horizon semantics of $\tilde{\varphi}$.
\end{condproof}
\begin{condsection}[lem:psi_dp]
\begin{tcolorbox}[ThmFrame3]
\begin{lemma} \label{lem:psi_dp}
    Let $\timer{}_s \coloneqq \optwitness{\psi}(x_{0:s}) - s$, and assume $\timer{}_s \geq 0$.
    Let $\bar{x}_{0:\infty}$ be the trajectory generated by $\pi_{[\psi]}$ from $x_{0:s}$.
    Then,
    \begin{equation}
        \rho_{[\tilde{\psi}]}([\bar{x}_{s:\infty}, \timer{}_s])
        = \Vopt{\tilde{\psi}}([\bar{x}_s, \timer{}_s])
    \end{equation}
\end{lemma}
\end{tcolorbox}
\end{condsection}
\begin{condproof}
Let
$ n^* \coloneqq \optwitness{\psi}(x_{0:s})$ and
$ \timer{}_s \coloneqq n^*-s \geq 0$.
By \Cref{lem:replace_r_with_Vr},
$ n^*=\optwitness{\varphi}(x_{0:s}) $.
Let $\hat{x}_{0:\infty}$ and $\bar{x}_{0:\infty}$ be the trajectories
generated from $x_{0:s}$ by $\pi_{[\varphi]}$ and $\pi_{[\psi]}$,
respectively.
By \Cref{lem:pi_psi_matches_pi_varphi_before_switch},
\begin{equation} \label{eq:psi_dp:same_prefix}
    \bar{x}_{0:n^*}=\hat{x}_{0:n^*}.
\end{equation}

As established in the proof of \Cref{thm:simple_until_policy_optimal},
the trajectory $[\hat{x}_{s:\infty},\timer{}_s]$ is optimal for
$\tilde{\varphi}$ from $x_{0:s}$.
Since $\bar{x}$ and $\hat{x}$ agree on the first $n^*$ states,
\Cref{lem:until_Vr_finite_prefix} implies that
$[\bar{x}_{s:\infty},\timer{}_s]$ is also optimal for $\tilde{\varphi}$.
Hence
\begin{equation} \label{eq:psi_dp:tilde_varphi_opt}
    \rho_{[\tilde{\varphi}]}([\bar{x}_{s:\infty},\timer{}_s])
    =
    \Vopt{\tilde{\varphi}}([\bar{x}_s,\timer{}_s]).
\end{equation}
By \Cref{lem:late_optimal_witness_suffix_value}, if we define
$
c^* \coloneqq \smin{s\le i<n^*} q(\bar{x}_i)
$,
then
\begin{equation} \label{eq:psi_dp:tilde_varphi_repr}
    \rho_{[\tilde{\varphi}]}([\bar{x}_{s:\infty},\timer{}_s])
    =
    \min\{c^*,\Vopt{\mathsf r}(\bar{x}_{n^*})\}.
\end{equation}
Next, since $\hat{x}$ is generated by $\pi_{[\varphi]}$,
\Cref{lem:witness_simulates_timer} gives
\[
\optwitness{\varphi}(\hat{x}_{0:n^*})=n^*.
\]
Using \Cref{lem:replace_r_with_Vr}, we obtain
$ \optwitness{\psi}(\bar{x}_{0:n^*})=n^* $.
Applying \Cref{lem:witness_fixed_after_switch} at time $n^*$ yields
\[
\optwitness{\psi}(\bar{x}_{0:k})=n^*,
\qquad \forall k\ge n^*.
\]
Therefore, by the definition of $\pi_{[\psi]}$,
$
\pi_{[\psi]}(\bar{x}_{0:k})
=
\pi_{[\mathsf r]}(\bar{x}_{n^*:k})
$
for all $k \geq n^*$.
Thus, the suffix $\bar{x}_{n^*:\infty}$ is exactly the trajectory generated by
$\pi_{[\mathsf r]}$ from $\bar{x}_{n^*}$.
By the optimality of $\pi_{[\mathsf r]}$,
$
\rho_{[\mathsf r]}(\bar{x}_{n^*:\infty})
= \Vopt{\mathsf r}(\bar{x}_{n^*})
$.
Now evaluate $\tilde{\psi}$ at the candidate witness time
$\timer{}_s=n^*-s$.
Using \Cref{lem:finite_horizon_semantics} (with $t = n^* \leq \timer{}_s$) and
\eqref{eq:psi_dp:tilde_varphi_repr},
\begin{align}
    \rho_{[\tilde{\psi}]}([\bar{x}_{s:\infty},\timer{}_s])
&\ge
    \min\{c^*,\rho_{[\mathsf r]}(\bar{x}_{n^*:\infty})\} \notag\\
    &=
    \min\{c^*,\Vopt{\mathsf r}(\bar{x}_{n^*})\}
    =
    \rho_{[\tilde{\varphi}]}([\bar{x}_{s:\infty},\timer{}_s]).
    \label{eq:psi_dp:lower}
    \raisetag{6ex}
\end{align}
Now, since
$
\rho_{[\mathsf r]}(\bar{x}_{s+t:\infty})
\le
\Vopt{\mathsf r}(\bar{x}_{s+t})
$ for any $t \in [0, \timer{}_s]$,
each term of $\rho_{[\tilde{\psi}]}$ is bounded above by the corresponding term in $\rho_{[\tilde{\varphi}]}$.
We thus have
\begin{equation} \label{eq:psi_dp:upper}
    \rho_{[\tilde{\psi}]}([\bar{x}_{s:\infty},\timer{}_s])
    \le
    \rho_{[\tilde{\varphi}]}([\bar{x}_{s:\infty},\timer{}_s]).
\end{equation}
Combining \eqref{eq:psi_dp:lower} and \eqref{eq:psi_dp:upper},
\[
\rho_{[\tilde{\psi}]}([\bar{x}_{s:\infty},\timer{}_s])
=
\rho_{[\tilde{\varphi}]}([\bar{x}_{s:\infty},\timer{}_s]).
\]
Using \eqref{eq:psi_dp:tilde_varphi_opt} and \Cref{lem:replace_r_with_Vr},
we conclude that
\[
\rho_{[\tilde{\psi}]}([\bar{x}_{s:\infty},\timer{}_s])
=
\Vopt{\tilde{\varphi}}([\bar{x}_s,\timer{}_s])
=
\Vopt{\tilde{\psi}}([\bar{x}_s,\timer{}_s]).
\]
This proves the claim.
\end{condproof}
\begin{tcolorbox}[ThmFrame3]
\begin{thm} \label{thm:pi_psi_optimal}
    $\pi_{[\psi]}$ is an optimal policy for $\psi$.
\end{thm}
\end{tcolorbox}
\begin{tcolorbox}[AssmpFrame2]
\begin{proof}[Proof Sketch]
Fix a history $x_{0:s}$ and let
$n^* \coloneqq \optwitness{\psi}(x_{0:s})$.
The proof splits into two cases.

\noindent\textbf{($\bm{n^* \ge s}$):} The optimal witness has not yet occurred.
From time $s$ onward, the problem is equivalent to a finite-horizon Until
problem with remaining horizon $n^*-s$.
Since $\pi_{[\psi]}$ matches $\pi_{[\varphi]}$ until the witness is reached,
we use the optimality of $\pi_{[\varphi]}$ (\Cref{thm:simple_until_policy_optimal}) and the equivalence in value of $\psi$ and $\varphi$ (\Cref{lem:replace_r_with_Vr}) to conclude that $\pi_{[\psi]}$ also attains the optimal value for this finite-horizon problem.
When the witness time $n^*$ is reached, the policy switches to the suffix controller $\pi_{[\mathsf r]}$ which is optimal for $\mathsf r$.
After showing that $n_k - k$ correctly simulates the timer state, we conclude that the suffix after time $n^*$ is exactly the trajectory generated by $\pi_{[\mathsf r]}$ from $x_{n^*}$, and therefore attains the optimal $\mathsf r$-value.
Finally, the full trajectory attains the optimal value $\Vhist{\psi}(x_{0:s})$ by the history-level decomposition of Until.
\smallskip

\noindent\textbf{($\bm{n^* < s}$):} The optimal witness already lies in the fixed prefix.
Thus the Until objective has already switched to the suffix problem for
$\mathsf r$ at time $n^*$.
By construction, $\pi_{[\psi]}$ follows the optimal policy for $\mathsf r$ from
then on, so the resulting suffix achieves the optimal $\mathsf r$-value.
Combining this with the fixed prefix contribution
$\min_{0\le i<n^*} q(x_i)$ shows again that the full trajectory attains
$\Vhist{\psi}(x_{0:s})$.

Combining both results yields the desired result.
\end{proof}
\end{tcolorbox}
\begin{condsection}[thm:pi_psi_optimal]
\begin{tcolorbox}[AssmpFrame2]
\begin{proof}
Fix a history $\bar{x}_{0:s}\in\mathcal X^+$ with $s\ge 0$, and let
$\bar{x}_{0:\infty}$ be the trajectory generated by $\pi_{[\psi]}$ from
$\bar{x}_{0:s}$.
Define
$
n^* \coloneqq \optwitness{\psi}(\bar{x}_{0:s})$,
$\timer{}_s \coloneqq n^*-s$,
and $\timer{}_0 \coloneqq n^*$.
We will prove that
\begin{equation} \label{eq:thm_pi_psi_optimal:main_claim}
    \rho_{[\psi]}(\bar{x}_{0:\infty})=\Vhist{\psi}(\bar{x}_{0:s}).
\end{equation}
We consider two cases.

\medskip
\noindent\textbf{Case 1: $n^*\ge s$ (i.e., $\timer{}_s\ge 0$).}
By \Cref{lem:finite_infinite_horizon_Q} and \eqref{eq:finite_horizon_ineq},
\begin{equation} \label{eq:thm_pi_psi_optimal:basic_chain_short}
    \rho_{[\tilde{\psi}]}([\bar{x}_{0:\infty},\timer{}_0])
    \le
    \rho_{[\psi]}(\bar{x}_{0:\infty})
    \le
    \Vhist{\psi}(\bar{x}_{0:s}).
\end{equation}
Let $ c_s \coloneqq \smin{0\le i<s} q(\bar{x}_i) $.
By \Cref{lem:replace_r_with_Vr} and
\Cref{lem:late_optimal_witness_suffix_value},
\begin{equation} \label{eq:thm_pi_psi_optimal:case1_value_repr_short}
    \Vhist{\psi}(\bar{x}_{0:s})
    =
    \min\{c_s,\Vopt{\tilde{\psi}}([\bar{x}_s,\timer{}_s])\}.
\end{equation}
Applying \Cref{lem:history_dp} to the actual trajectory $\bar{x}_{0:\infty}$ gives
\[
\rho_{[\tilde{\psi}]}([\bar{x}_{0:\infty},\timer{}_0])
\ge
\min\{c_s,\rho_{[\tilde{\psi}]}([\bar{x}_{s:\infty},\timer{}_s])\}.
\]
By \Cref{lem:psi_dp},
$\rho_{[\tilde{\psi}]}([\bar{x}_{s:\infty},\timer{}_s])
= \Vopt{\tilde{\psi}}([\bar{x}_s,\timer{}_s])$.
Hence
\begin{equation*}
\rho_{[\tilde{\psi}]}([\bar{x}_{0:\infty},\timer{}_0])
\ge
\min\{c_s,\Vopt{\tilde{\psi}}([\bar{x}_s,\timer{}_s])\}
\stackrel{\smash{\eqref{eq:thm_pi_psi_optimal:case1_value_repr_short}}}{=}
\Vhist{\psi}(\bar{x}_{0:s}).
\end{equation*}
Combined with \eqref{eq:thm_pi_psi_optimal:basic_chain_short}, this yields \eqref{eq:thm_pi_psi_optimal:main_claim}.

\medskip
\noindent\textbf{Case 2: $n^*<s$ (i.e., $\timer{}_s<0$).}
Let
$\tau \coloneqq n^*=\optwitness{\psi}(\bar{x}_{0:s})$ and
$c_\tau \coloneqq \smin{0\le i<\tau} q(\bar{x}_i)$.
By \Cref{lem:past_witness_value_identity},
\begin{equation} \label{eq:thm_pi_psi_optimal:case2_value}
    \Vhist{\psi}(\bar{x}_{0:s})
    =
    \min\{c_\tau,\Vhist{\mathsf r}(\bar{x}_{\tau:s})\}.
\end{equation}
Moreover, by \Cref{item:witness_fixed_after_switch:suffix},
the suffix $\bar{x}_{\tau:\infty}$ is exactly the trajectory generated by
$\pi_{[\mathsf r]}$ from $\bar{x}_{\tau:s}$, and
$
\rho_{[\mathsf r]}(\bar{x}_{\tau:\infty})
= \Vhist{\mathsf r}(\bar{x}_{\tau:s}).
$
Furthermore, by \Cref{item:witness_fixed_after_switch:rho},
\begin{equation} \label{eq:thm_pi_psi_optimal:case2_lower}
    \rho_{[\psi]}(\bar{x}_{0:\infty})
    \ge
    \min\{c_\tau,\Vhist{\mathsf r}(\bar{x}_{\tau:s})\}.
\end{equation}
We get
$\rho_{[\psi]}(\bar{x}_{0:\infty})
\ge \Vhist{\psi}(\bar{x}_{0:s})$
by combining \eqref{eq:thm_pi_psi_optimal:case2_value} and
\eqref{eq:thm_pi_psi_optimal:case2_lower}.
By definition of $\Vhist{\psi}$,
$\rho_{[\psi]}(\bar{x}_{0:\infty})
\le \Vhist{\psi}(\bar{x}_{0:s})$.
Thus, \eqref{eq:thm_pi_psi_optimal:main_claim} holds.

In both cases, the trajectory generated by $\pi_{[\psi]}$ from the arbitrary
history $\bar{x}_{0:s}$ achieves the value $\Vhist{\psi}(\bar{x}_{0:s})$.
Therefore $\pi_{[\psi]}$ is optimal for $\psi$ in the sense of
\Cref{def:def_of_optimal_policy}.
\end{proof}
\end{tcolorbox}
\end{condsection}

\ifcdc
We omit the proof for $\hat{\pi}_{[\psi]}$ but it follows similarly.
\fi
\begin{tcolorbox}[ThmFrame2]
\begin{thm} \label{thm:alternative_pi_psi_optimal}
    $\hat{\pi}_{[\psi]}$ is an optimal policy for $\psi$.
\end{thm}
\end{tcolorbox}
\begin{condproof}[thm:alternative_pi_psi_optimal]
The proof is similar to the proof of \Cref{thm:alternative_until_simple_optimal_policy_optimal}, except we now prove
\begin{equation} \label{eq:proof:alternative_pi_psi_optimal:induction}
    \rho_{[\psi]}(\bar{x}_{0:\infty}) = \Vopt{\psi}(\bar{x}_0)      
\end{equation}
via induction on $\tau(\bar{x}_0)$.

\noindent\textbf{Base case ($\tau(\bar{x}_0) = 0$)}: $\hat{\pi}_{[\psi]}$ follows $\pi_{[\mathsf{r}]}$ from $\bar{x}_0$, and by the optimality of $\pi_{[\mathsf{r}]}$,
\begin{equation}
    \rho_{[\psi]}(\bar{x}_{0:\infty})
    = \rho_{[\mathsf{r}]}(\bar{x}_{0:\infty})
    \stackrel{\mathclap{\eqref{eq:pi_r_optimal}}}{=} \Vopt{\mathsf{r}}(\bar{x}_0)
    = \Vopt{\psi_0}(\bar{x}_0)
    = \Vopt{\psi}(\bar{x}_0).
\end{equation}

\noindent\textbf{Induction step}: Assume \eqref{eq:proof:alternative_pi_psi_optimal:induction} holds for all $x_0$ where $\tau(x_0) < \tau(\bar{x}_0)$ and $\tau(\bar{x}_0) > 0$.
Since $\hat{\pi}_{[\psi]}$ follows $\hat{\pi}_{[\varphi]}$ (which is Markovian) for $k < n^* = \tau(\bar{x}_0)$, and follows $\pi_{[\mathsf{r}]}$ for $k \geq n^*$, the suffix trajectory $\bar{x}_{1:\infty}$ is generated by $\hat{\pi}_{[\psi]}$ from $\bar{x}_1$ with $\tau(\bar{x}_1) = \tau(\bar{x}_0) - 1$. By the induction hypothesis,
$\rho_{[\psi]}(\bar{x}_{1:\infty}) = \Vopt{\psi}(\bar{x}_1)$.
By the recursive characterization of $\mathsf{U}$ and the induction hypothesis,
\begin{align*}
    \rho_{[\psi]}(\bar{x}_{0:\infty})
    &= \mathsf{r}(\bar{x}_{0:\infty}) \lor \bigl( q(\bar{x}_0) \land \rho_{[\psi]}(\bar{x}_{1:\infty}) \bigr), \\
    &= \mathsf{r}(\bar{x}_{0:\infty}) \lor \bigl( q(\bar{x}_0) \land \Vopt{\psi}(\bar{x}_1) \bigr).
\end{align*}
The upper bound $\rho_{[\psi]}(\bar{x}_{0:\infty}) \leq \Vopt{\psi}(\bar{x}_0)$ is immediate. It remains to show the lower bound.

Since $\tau(\bar{x}_0) > 0$, we have $n^*_0 = \optwitness{\varphi}(\bar{x}_0) > 0$. By Lemma \ref{lem:until_Vr_Vr_leq_rhs_varphi} and Lemma \ref{lem:replace_r_with_Vr},
\begin{equation}
   \Vopt{\mathsf{r}}(\bar{x}_0) \leq q(\bar{x}_0) \land \max_a \Vopt{\varphi_{n^*_0 - 1}}(f(\bar{x}_0, a)), 
\end{equation}
so the action $\bar{a}_0 = \hat{\pi}_{[\varphi]}(\bar{x}_0)$ satisfies
\begin{equation}
   \Vopt{\psi}(\bar{x}_0) = \Vopt{\varphi}(\bar{x}_0) = q(\bar{x}_0) \land \Vopt{\varphi_{n^*_0 - 1}}(\bar{x}_1).
\end{equation}
Since $\Vopt{\psi}(\bar{x}_1) = \Vopt{\varphi}(\bar{x}_1) \geq \Vopt{\varphi_{n^*_0 - 1}}(\bar{x}_1)$ by \eqref{eq:finite_horizon_ineq} and Lemma \ref{lem:replace_r_with_Vr},
\begin{equation}
   q(\bar{x}_0) \land \Vopt{\psi}(\bar{x}_1) \geq q(\bar{x}_0) \land \Vopt{\varphi_{n^*_0 - 1}}(\bar{x}_1) = \Vopt{\psi}(\bar{x}_0). 
\end{equation}
Therefore,
\begin{equation}
   \rho_{[\psi]}(\bar{x}_{0:\infty}) \geq q(\bar{x}_0) \land \Vopt{\psi}(\bar{x}_1) \geq \Vopt{\psi}(\bar{x}_0). 
\end{equation}
Combining both bounds, $\rho_{[\psi]}(\bar{x}_{0:\infty}) = \Vopt{\psi}(\bar{x}_0)$. The claim thus holds by induction.
\end{condproof}
\begin{tcolorbox}[NoteFrame2]
\begin{rmk}[Computing the policy without witness times] \label{rmk:implicit_witness_time}
    The optimal witness time $n^*=\optwitness{\psi}(x_0)$ (\Cref{def:optimal_witness}) need not be computed explicitly, since it is defined as the first time $\tau$ such that
    $
    \min\{ \Vopt{\mathsf{r}}(x_{\tau}), \sminraise{0 \leq k < \tau} q(x_k) \}
    \geq \min\{ \Vopt{\psi}(x_{\tau+1}), \sminraise{0 \leq k < \tau + 1} q(x_k) \}
    $ by \eqref{eq:witness_exist:1}.
    Thus, it suffices to track $\min_{0 \leq k < \tau} q(x_k)$ and check the above condition at each time step $\tau$.
    Similarly,
$\optwitness{\varphi}(x_0) = 0$
    if and only if
    $\Vopt{\mathsf{r}}(x_0) \geq \smax{a} q(x_0) \land \Vopt{\varphi}(f(x_0, a))$.
    Thus, it suffices to check the above condition at each time to determine when to ``switch'' to $\pi_{[\mathsf{r}]}$.
\end{rmk}
\end{tcolorbox}

We compare $\pi_{[\psi]}$ \eqref{eq:until_optimal_policy} and $\hat{\pi}_{[\psi]}$ \eqref{eq:alternative_until_optimal_policy} in \Cref{fig:compare_pi_pihat}.
$\pi_{[\psi]}$ requires tracking either the timer $\timer{}$ or the cumulative minimum $\min_{0 \leq k < \tau} q(x_k)$ per \Cref{rmk:implicit_witness_time}, but switches to $\pi_{{[\mathsf{r}]}}$ in the fewest steps.
In contrast, $\hat{\pi}_{[\psi]}$ only requires storing a single boolean to track whether $\optwitness{\varphi}(x_k) = 0$ has been reached, but potentially takes more steps to switch to $\pi_{[\mathsf{r}]}$.

\begin{figure}[t]
    \definecolor{FigRed}{HTML}{e24b34}
    \definecolor{FigBlue}{HTML}{338abd}

    \vspace{-4ex}
    \centering
\includegraphics[width=0.7\linewidth]{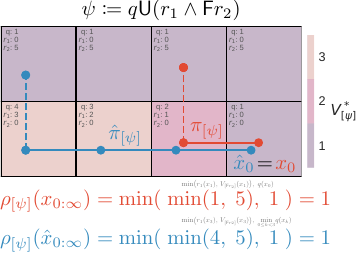}
\vspace{6pt}
    \caption{\textbf{Comparing $\textcolor{FigRed}{\pi_{[\psi]}}$ and $\textcolor{FigBlue}{\hat{\pi}_{[\psi]}}$}.
    Both $\textcolor{FigRed}{\pi_{[\psi]}}$ and $\textcolor{FigBlue}{\hat{\pi}_{[\psi]}}$ achieve the optimal robustness at $x_0 = \hat{x}_0$.
    However, $\textcolor{FigRed}{\pi_{[\psi]}}$ achieves this in 2 steps, while
    $\textcolor{FigBlue}{\hat{\pi}_{[\psi]}}$ reaches a higher $r_1(\hat{x}_3) = 3$, albeit taking 4 steps. 
}
   \vspace{-1.9em}
   \label{fig:compare_pi_pihat}
\end{figure}

\section{Policies for Next, Disjunction and Conjunction}
Next, we construct optimal policies for the $\LTLNext$, $\lor$ and $\land$ operators to tackle \Cref{def:S:Next,def:S:Disjunction,def:S:Conjunction}.
Fortunately, these operators do not suffer from the deferral issue of $\LTLUntil$ as in \Cref{ex:greedy_Q_suboptimal} and thus have more straightforward optimal policies when the operands are propositional.
However, in the general case where the operands are TL formulas, the optimal policies can be more complex.
To handle this, we use the same technique as in \Cref{sec:until:general_r} and assume the existence of optimal policies for the operand formulas (see \Cref{assmp:exists_optimal_policy_for_r}).

\begin{tcolorbox}[DefFrame2]
\begin{defi} \label{def:policy_next}
    Let $\pi_{[\mathsf{r}]}$ be an optimal policy for $\mathsf{r}$. 
    Then, define $\pi_{[\LTLNext \mathsf{r}]}$ as
    \begin{equation}
        \pi_{[\LTLNext \mathsf{r}]}(x_{0:k}) \coloneqq \begin{dcases}
            \argmax_{a \in \calA} \Vopt{\mathsf{r}}(f(x_0, a)), & k = 0, \\
            \pi_{[\mathsf{r}]}(x_{1:k}), & k > 0.
        \end{dcases}
    \end{equation}
\end{defi}
\end{tcolorbox}
\begin{tcolorbox}[DefFrame2]
\begin{defi} \label{def:disjunction_policy}
    For set $I$, let $\pi_{[\mathsf{r}_i]}$ be an optimal policy for $\mathsf{r}_i$ for each $i \in I$.
    Define the policy $\pi_{[\bigvee_{i \in I} \mathsf{r}_i]}$ as
    \begin{equation} \label{eq:disjunction_policy}
        \pi_{[\bigvee_{i \in I} \mathsf{r}_i]}(x_{0:k}) \coloneqq
        \pi_{[\mathsf{r}_{i^*}]}(x_{0:k}),
    \end{equation}
    where $i^* \in \argmax_{i \in I} \Vopt{\mathsf{r}_i}(x_0)$ is any 
maximizer.
\end{defi}
\end{tcolorbox}
\begin{tcolorbox}[DefFrame2]
\begin{defi} \label{def:conjunction_policy}
    Let $\pi_{[\mathsf{r}]}$ be an optimal policy for $\mathsf{r}$. Define the policy $\pi_{[q \land \mathsf{r}]}$ as
    \begin{equation} \label{eq:conjunction_policy}
        \pi_{[q \land \mathsf{r}]} = \pi_{[\mathsf{r}]}.
    \end{equation}
\end{defi}
\end{tcolorbox}
\begin{tcolorbox}[ThmFrame2]
\begin{thm} \label{thm:next_disjunction_conjunction_optimal_policies}
    The policies defined in \Cref{def:policy_next,def:disjunction_policy,def:conjunction_policy} are optimal for $\LTLNext \mathsf{r}$, $\bigvee_{i \in I} \mathsf{r}_i$ and $q \land \mathsf{r}$, respectively.
\end{thm}
\end{tcolorbox}

\section{Policy for Globally}
Finally, we construct an optimal policy for the $\LTLGlobally$ operator, another important operator in TL specifications due to its ability to express safety and invariance properties \cite{baier2008principles}.

\subsection{Globally with Propositional Operand}
When the operand is propositional (\Cref{def:S:Globally}), the greedy Markovian policy is optimal.
\begin{tcolorbox}[DefFrame2]
\begin{defi} \label{def:greedy_Gq}
    Define the Markovian policy $\pi_{[\LTLGlobally q]} : \calX \to \calA$ as
    \begin{equation} \label{eq:greedy_Gq}
        \pi_{[\LTLGlobally q]}(x) \coloneqq \argmax_{a \in \calA}\, \Qopt{\LTLGlobally q}(x, a).
    \end{equation}
\end{defi}
\end{tcolorbox}
\begin{condsection}[Gq_optimal_lemma]
We first establish that the value function for $\LTLGlobally q$ satisfies a Bellman-type recursion.

\begin{lemma} \label{lem:Gq_bellman_recursion}
    For all $x \in \calX$,
    \begin{equation} \label{eq:Gq_bellman_recursion}
        \Vopt{\LTLGlobally q]}(x) = \max_{a \in \calA}\, \min\bigl( q(x),\; \Vopt{\LTLGlobally q]}(f(x, a)) \bigr).
    \end{equation}
\end{lemma}

\begin{condproof}[lem:Gq_bellman_recursion]
    By definition,
    \begin{align*}
        \Vopt{\LTLGlobally q]}(x)
        &= \max_{a_{0:\infty}} \inf_{k \geq 0}\, q(x_k), \\
        &= \max_{a_0} \max_{a_{1:\infty}} \min\bigl( q(x_0),\; \inf_{k \geq 1}\, q(x_k) \bigr), \\
        &= \max_{a_0}\, \min\bigl( q(x_0),\; \max_{a_{1:\infty}} \inf_{k \geq 1}\, q(x_k) \bigr), \\
        &= \max_{a_0}\, \min\bigl( q(x),\; \Vopt{\LTLGlobally q]}(f(x, a_0)) \bigr). \qedhere
    \end{align*}
\end{condproof}

\begin{lemma} \label{lem:Gq_value_nondecreasing}
    Let $\bar{x}_{0:\infty}$ be the trajectory generated by $\pi_{[\LTLGlobally q]}$ from $\bar{x}_0$. Then:
    \begin{enumerate}
        \item[(i)] For all $k \geq 0$,\quad $\Vopt{\LTLGlobally q]}(\bar{x}_k) = \min\bigl( q(\bar{x}_k),\; \Vopt{\LTLGlobally q]}(\bar{x}_{k+1}) \bigr)$.
        \item[(ii)] The sequence $\bigl( \Vopt{\LTLGlobally q]}(\bar{x}_k) \bigr)_{k \geq 0}$ is non-decreasing.
        \item[(iii)] For all $k \geq 0$,\quad $q(\bar{x}_k) \geq \Vopt{\LTLGlobally q]}(\bar{x}_0)$.
    \end{enumerate}
\end{lemma}

\begin{condproof}[lem:Gq_value_nondecreasing]
    \textbf{(i)}: Since
\begin{equation}
    \bar{a}_k = \pi_{[\LTLGlobally q]}(\bar{x}_k) = \argmax_{a} \min( q(\bar{x}_k), \Vopt{\LTLGlobally q]}(f(\bar{x}_k, a)) ),
    \end{equation}
    the greedy action attains the maximum in \eqref{eq:Gq_bellman_recursion}.

    \textbf{(ii)}: From (i), $\Vopt{\LTLGlobally q]}(\bar{x}_k) \leq \Vopt{\LTLGlobally q]}(\bar{x}_{k+1})$ since $\Vopt{\LTLGlobally q]}(\bar{x}_k)$ is the minimum of two quantities, one of which is $\Vopt{\LTLGlobally q]}(\bar{x}_{k+1})$.

    \textbf{(iii)}: From (i), $q(\bar{x}_k) \geq \Vopt{\LTLGlobally q]}(\bar{x}_k)$. Combining with (ii), $q(\bar{x}_k) \geq \Vopt{\LTLGlobally q]}(\bar{x}_k) \geq \Vopt{\LTLGlobally q]}(\bar{x}_0)$.
\end{condproof}
\end{condsection}
\begin{tcolorbox}[ThmFrame2]
\begin{thm} \label{thm:Gq_greedy_optimal}
    $\pi_{[\LTLGlobally q]}$ is an optimal policy for $\LTLGlobally q$.
\end{thm}
\end{tcolorbox}
\begin{condproof}[thm:Gq_greedy_optimal]
    The upper bound $\rho_{[\LTLGlobally q]}(\bar{x}_{0:\infty}) \leq \Vopt{\LTLGlobally q]}(\bar{x}_0)$ is immediate from the definition of $\Vopt{\LTLGlobally q]}$.
    For the lower bound, by \Cref{lem:Gq_value_nondecreasing}(iii), $q(\bar{x}_k) \geq \Vopt{\LTLGlobally q]}(\bar{x}_0)$ for all $k \geq 0$.
    Thus,
    \begin{equation}
        \rho_{[\LTLGlobally q]}(\bar{x}_{0:\infty})
        = \inf_{k \geq 0}\, q(\bar{x}_k)
        \geq \Vopt{\LTLGlobally q]}(\bar{x}_0).
    \end{equation}
    Combining both bounds yields $\rho_{[\LTLGlobally q]}(\bar{x}_{0:\infty}) = \Vopt{\LTLGlobally q]}(\bar{x}_0)$.
\end{condproof}

\subsection{Globally with Conjunctions of Until operands}
Finally, we tackle the case of $\LTLGlobally \LTLUntil$ (\Cref{def:S:GU}).
When the operand contains the $\LTLUntil$ operator, this encodes a Büchi acceptance criterion that requires infinitely many visits to a set of states \cite{vardi1986automata}, which requires a more complex policy to solve the deferral issue as in \Cref{ex:greedy_Q_suboptimal}.
Using the recursive relationships for $\LTLGlobally$, we have that
\begin{equation}
    \LTLGlobally( q \LTLUntil r ) \equiv (q \LTLUntil r) \land \LTLNext \LTLGlobally( q \LTLUntil r ).
\end{equation}
Thus, $\LTLGlobally \LTLUntil$ can be viewed as $\LTLUntil$ with an infinite tail of $\LTLGlobally \LTLUntil$.
Consequently, a naively constructed policy for $\LTLGlobally \LTLUntil$ by maximizing the Q function can also result in a policy that infinitely defers satisfying the $\LTLUntil$ component, and thus fails to satisfy the $\LTLGlobally \LTLUntil$ formula.

Let $I$ be a nonempty finite set. We now consider $\psi \coloneqq \LTLGlobally( \bigwedge_{i \in I} q_i \LTLUntil r_i )$.
Without loss of generality, we take $I = \{1, 2, \dots, N\}$.
The following lemma, taken from \cite{sharpless2026bellman}, makes precise this recursive structure of $\LTLGlobally \LTLUntil$.
\begin{tcolorbox}[ThmFrame2]
\begin{lemma}[{\cite[Theorem 3]{sharpless2026bellman}}] \label{lem:gu_recursive_single}
    For any $i \in I$,
    \begin{equation} \label{eq:gu_recursive_single}
        \psi
\equiv \tilde{q}_i \LTLUntil (\tilde{r}_i \land \LTLNext \psi),
\;\; \tilde{q}_i \coloneqq q_i \land w_i,\;\; \tilde{r}_i \coloneqq r_i \land w_i
    \end{equation}
where $w_i \coloneqq \bigwedge_{j \in I \setminus \{i\}} (q_j \lor r_j)$.
\end{lemma}
\end{tcolorbox}
\noindent Using \Cref{lem:gu_recursive_single}, we can also show the following result.
\begin{tcolorbox}[ThmFrame2]
\begin{coroll} \label{lem:gu_recursive_nested}
    Let $\chi_i[\cdot]$ denote the RHS of \eqref{eq:gu_recursive_single} but with $\psi$ replaced by an arbitrary formula, i.e.,
    \begin{equation}
\chi_i[\theta] \coloneqq \tilde{q}_i \LTLUntil (\tilde{r}_i \land \LTLNext\theta).
    \end{equation}
    Then, $\chi_i[\psi] \equiv \psi$ for all $i \in I$, and
    \begin{equation} \label{eq:gu_recursive_nested}
        \psi \equiv \chi_1[\psi] \equiv \chi_1[\chi_2[\cdots \chi_N[\psi] \cdots]].
    \end{equation}
\end{coroll}
\end{tcolorbox}
\begin{condproof}[lem:gu_recursive_nested]
    The result follows by repeatedly applying \Cref{lem:gu_recursive_single} to the RHS of \eqref{eq:gu_recursive_single}.
\end{condproof}
We now construct an optimal policy $\pi_{[\psi]}$ for $\psi$ by making use of this infinite nested structure in \Cref{lem:gu_recursive_nested}.
Specifically, we will iteratively visit each $\LTLUntil$ component in sequence, then repeat this process infinitely many times, which ensures that we satisfy the Büchi acceptance criteria of $\psi$.

\begin{tcolorbox}[NoteFrame2]
\begin{note}
    The order that the $\LTLUntil$ components are visited can be arbitrary (see \cite{sharpless2026bellman}).
    We use ascending order for simplicity.
\end{note}
\end{tcolorbox}
\begin{tcolorbox}[DefFrame2]
\begin{defi} \label{def:gu_optimal_policy}
    For $i \in I$, let the \textbf{until} formula $\phi_i$ denote the $i$-th nested $\LTLUntil$ formula in \Cref{lem:gu_recursive_nested}, i.e.,
    \begin{equation}
        \phi_i
        \coloneqq \chi_i[\chi_{i+1}[\cdots \chi_N[\psi] \cdots]]
        = (q_i \land w_i) \LTLUntil (r_i \land w_i \land \phi_{i+1}).
    \end{equation}
Now, define the policy $\pi{[\psi]}$ as
    \begin{equation}
        \pi_{[\psi]}(x_{0:k}) \coloneqq \pi_{[\phi_1]}(x_{0:k}),
    \end{equation}
    \ifcdc
where we use the optimal policy for $\LTLUntil$ (\eqref{eq:until_optimal_policy} or \eqref{eq:alternative_until_optimal_policy}), since $\phi_1$ is a (deeply nested) $\LTLUntil$ formula.
\fi
\end{defi}
\end{tcolorbox}
\begin{tcolorbox}[ThmFrame2]
\begin{thm} \label{thm:gu_optimal}
    $\pi_{[\psi]}$ from \Cref{def:gu_optimal_policy} is an optimal policy.
\end{thm}
\end{tcolorbox}

\begin{condproof}[thm:gu_optimal]
    Fix $\bar{x}_0 \in \calX$ and let $\bar{x}_{0:\infty}$ be the trajectory generated by $\pi_{[\psi]}$ from $\bar{x}_0$.
    The upper bound $\rho_{[\psi]}(\bar{x}_{0:\infty}) \leq \Vopt{\psi}(\bar{x}_0)$ is immediate from the definition of $\Vopt{\psi}$.
    For the lower bound:
    \begin{equation}
        \rho_{[\psi]}(\bar{x}_{0:\infty})
        = \inf_{k \geq 0}\; \min_{\ell \in I}\; \rho_{[q_\ell \LTLUntil r_\ell]}(\bar{x}_{k:\infty}),
    \end{equation}
    \Cref{lem:gu:until_lower_bound} gives $\rho_{[q_\ell \LTLUntil r_\ell]}(\bar{x}_{k:\infty}) \geq \Vopt{\psi}(\bar{x}_0)$ for all $k \geq 0$ and $\ell \in I$.
    Thus,
    \begin{equation}
        \rho_{[\psi]}(\bar{x}_{0:\infty})
        = \inf_{k \geq 0}\; \min_{\ell \in I}\; \rho_{[q_\ell \LTLUntil r_\ell]}(\bar{x}_{k:\infty})
        \geq \Vopt{\psi}(\bar{x}_0).
    \end{equation}
    Combining both bounds, $\rho_{[\psi]}(\bar{x}_{0:\infty}) = \Vopt{\psi}(\bar{x}_0)$.
\end{condproof}

\section{Optimal Policies for Compositions} \label{sec:composition}
In the previous sections, we have established optimal policies for $\LTLUntil$
(\Cref{thm:pi_psi_optimal,thm:alternative_pi_psi_optimal}), 
$\LTLNext$, disjunctions and conjunctions with a single temporal operand (\Cref{thm:next_disjunction_conjunction_optimal_policies}),
for $\LTLGlobally$ (\Cref{thm:Gq_greedy_optimal}), and  $\LTLGlobally \LTLUntil$ (\Cref{thm:gu_optimal}).
These correspond exactly to the items in the definition of the fragment $\mathcal{S}$ in \Cref{def:solvable_fragment}.
Consequently, combining these results enables us to construct optimal policies for any formula in $\mathcal{S}$ via structural induction:
\begin{tcolorbox}[ThmFrame2]
\begin{thm}[Optimal policies for $\mathcal{S}$] \label{thm:solvable_fragment}
    For $\varphi \in \mathcal{S}$, the policy $\pi_{[\varphi]}$ constructed with \Cref{thm:pi_psi_optimal,thm:alternative_pi_psi_optimal,thm:next_disjunction_conjunction_optimal_policies,thm:Gq_greedy_optimal,thm:gu_optimal} is optimal (\Cref{def:def_of_optimal_policy}) and satisfies \eqref{eq:optimality_eq}.
\end{thm}
\end{tcolorbox}

\begin{condproof}[thm:solvable_fragment]
    By structural induction on $\mathcal{S}$.
    \begin{enumerate}
        \item[\textbf{(S1)}] When $\varphi \in \mathrm{Prop}$, the robustness score depends only on the current state, so any policy is optimal (\Cref{rmk:optimal_policy_propositional}).

        \item[\textbf{(S2)}] Let $\varphi = p \LTLUntil \mathsf{r}$ with $p \in \mathrm{Prop}$ and $\mathsf{r} \in \mathcal{S}$.
        By the inductive hypothesis, there exists an optimal policy $\pi_{[\mathsf{r}]}$ for $\mathsf{r}$, so \Cref{assmp:exists_optimal_policy_for_r} is satisfied.
        The result follows from \Cref{thm:pi_psi_optimal} (or \Cref{thm:alternative_pi_psi_optimal}).

        \item[\textbf{(S3)}] Let $\varphi = \LTLNext \mathsf{r}$ with $\mathsf{r} \in \mathcal{S}$.
        By the inductive hypothesis, there exists an optimal policy $\pi_{[\mathsf{r}]}$ for $\mathsf{r}$.
        The result follows from \Cref{thm:policy_next_optimal}.

        \item[\textbf{(S4)}] Let $\varphi = \LTLGlobally\bigl(\bigwedge_{i=1}^N p_i \LTLUntil r_i\bigr)$ with $p_i, r_i \in \mathrm{Prop}$.
        The result follows from \Cref{thm:gu_optimal}.

        \item[\textbf{(S5)}] Let $\varphi = \varphi_1 \lor \varphi_2$ with $\varphi_1, \varphi_2 \in \mathcal{S}$.
        By the inductive hypothesis, there exist optimal policies $\pi_{[\varphi_1]}$ and $\pi_{[\varphi_2]}$.
        The result follows from \Cref{thm:disjunction_optimal}.

        \item[\textbf{(S6)}] Let $\varphi = p \land \mathsf{r}$ with $p \in \mathrm{Prop}$ and $\mathsf{r} \in \mathcal{S}$.
        By the inductive hypothesis, there exists an optimal policy $\pi_{[\mathsf{r}]}$ for $\mathsf{r}$.
        The result follows from \Cref{thm:simple_conjunction_optimal}. \qedhere
    \end{enumerate}
\end{condproof}

\subsection{Expanding $\mathcal{S}$ by Rewriting Conjunctions}
A notable gap in the fragment $\mathcal{S}$ is the absence of conjunctions of multiple temporal operators, e.g., $\LTLFinally r_1 \land \LTLFinally r_2$.
However,
we can rewrite many such conjunctions into an equivalent single $\LTLUntil$ formula, which is in $\mathcal{S}$ by \textbf{(S2)} \cite{sharpless2026bellman}.
Namely, \cite[Lemma 7]{sharpless2026bellman} implies that conjunctions of $\LTLGlobally\LTLUntil$, $\LTLUntil$, and $\LTLGlobally$ formulas with propositional operands are equivalent to a formula in $\mathcal{S}$.

\begin{figure*}[t]
\vspace{-2ex}
   \centering
\includegraphics[width=\linewidth]{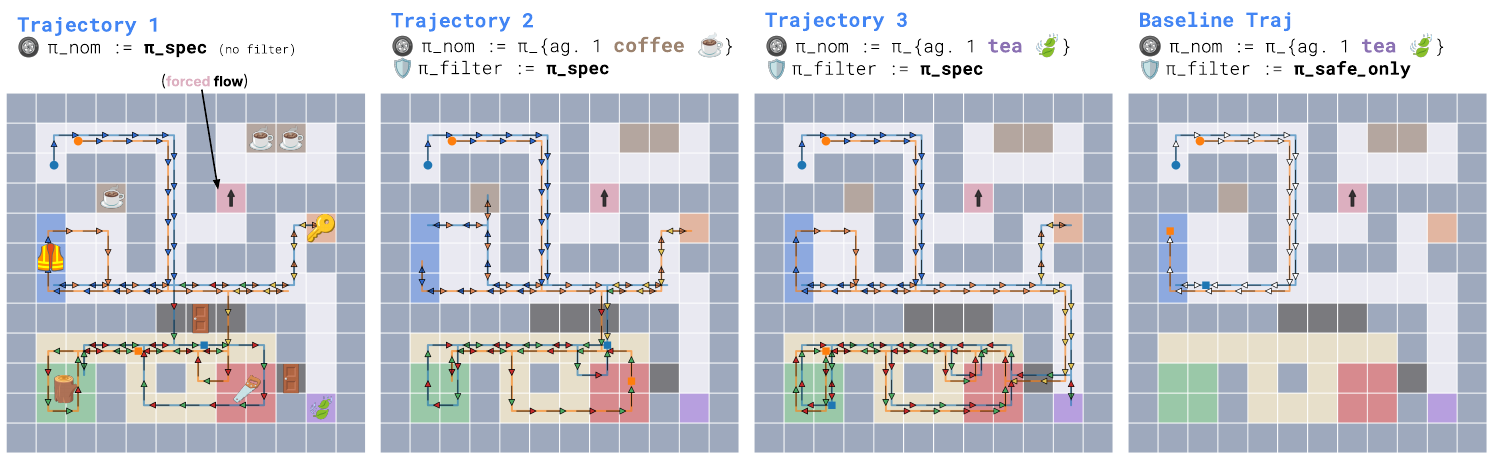}
   \caption{\textbf{Demonstration of optimal policy $\hat{\pi}$, independently and as a filter:} Here, we demonstrate the optimal policy for the value for a given specification (top) in a two-agent \texttt{GridWorld} system. The trajectory for the policy alone is given in the left plot, while the two center plots demonstrate its usage as a filter for two nominal policies to get \textbf{\textcolor{color_coffee}{coffee}} and \textbf{\textcolor{color_tea}{tea}} respectively. The plot on the far right demonstrates the problem with a "safety-only" filter here, which although safe is unable to yield a satisfactory trajectory.}
\vspace{-3ex}
   \label{fig:sim}
\end{figure*}

\begin{tcolorbox}[ThmFrame2]
\begin{coroll} \label{cor:master_formula_solvable}
    For finite sets $\IGUSet$ and $\IUSet$, and $q_i, \rgu_i, q_j, r_j, q \in \mathrm{Prop}$, the formula
    \begin{equation}
        \mathsf{p}_{\IGUSet, \IUSet} \coloneqq \bigwedge_{\smash{i \in \IGUSet}} \LTLGlobally(q_i\, \LTLUntil \, \rgu_i) \;\land\; \bigwedge_{\smash{j \in \IUSet}} (q_j\, \LTLUntil \, r_j) \;\land\; \LTLGlobally\, q
    \end{equation}
    is equivalent to a formula that is in $\mathcal{S}$.
\end{coroll}
\end{tcolorbox}
By \Cref{cor:master_formula_solvable}, $\mathcal{S}$ contains all conjunctions of $\LTLGlobally\LTLUntil$, $\LTLUntil$, and $\LTLGlobally$ formulas with propositional operands.
Moreover, the recursive structure of \textbf{(S2)} and \textbf{(S3)} allows nesting, so $\mathcal{S}$ contains a wide variety of formulas beyond those covered by \Cref{cor:master_formula_solvable}.
\begin{tcolorbox}[NoteFrame2]
\begin{rmk} \label{rmk:solvable_limitations}

    There are two limitations.
    (1) The absence of a general conjunction rule for two formulae.
Formulas such as $\LTLFinally\LTLGlobally\, p_1 \land \LTLGlobally\LTLFinally\, p_2$ are not \emph{syntactically} in $\mathcal{S}$.
    However, this is equivalent to single (albeit nested) $\LTLFinally$, i.e.,
$\LTLFinally\LTLGlobally\, p_1 \land \LTLGlobally\LTLFinally\, p_2
    \;\equiv\;
    \LTLFinally\bigl( \LTLGlobally( p_1 \LTLUntil (p_2 \land p_1) ) \bigr)$,
    and it can be verified that the right-hand side belongs to $\mathcal{S}$. (2) Nesting of temporal operators within $\LTLGlobally$ is limited.
    For example, $\LTLGlobally\bigl( p_1 \lor \LTLFinally(p_2 \land \LTLFinally p_3) \bigr)$
cannot be rewritten into an equivalent formula in $\mathcal{S}$.
\end{rmk}
\end{tcolorbox}

\section{Safety Filter}

Define the failure set $\mathcal{F} \subset \mathcal{X}^+$ as the set of state prefixes from which the specification $\psi$ cannot be satisfied with non-negative robustness score, i.e.,
\begin{equation} \label{eq:failure_set}
    \mathcal{F} \coloneqq \{ x_{0:k} \in \mathcal{X}^+ : \max_{a_{k:\infty}} \rho_{[\psi]}(x_{0:\infty}) < 0 \}.
\end{equation}
We now show that $Q$ constitutes a safety monitor \cite{hsu2023safety} for the TL formula $\psi$ under the fallback policy $\pi_{[\psi]}$.

\begin{tcolorbox}[ThmFrame2]
\begin{lemma} \label{thm:safety_monitor}
    Under the fallback policy $\pi_{[\psi]}$, $\Qopt{\psi}$ is a safety monitor for the specification $\psi$, i.e.,
    \begin{equation}
\Qopt{\psi}(x_{0:k}, a_k) \geq 0 \implies x_{0:k} \not\in \mathcal{F}.
\end{equation}
\end{lemma}
\end{tcolorbox}
\begin{tcolorbox}[AssmpFrame2]
\begin{proof}
    This follows by definition of Q \eqref{eq:Q} and $\mathcal{F}$ \eqref{eq:failure_set}.
\end{proof}
\end{tcolorbox}

Consequently, we can construct a least-restrictive safety filter \cite{hsu2023safety} which guarantees future satisfiability of $\psi$ under any task policy $\tilde{\pi}$ by only intervening when $\tilde{\pi}$ would violate the specification:

\begin{tcolorbox}[ThmFrame2]
\begin{thm} \label{thm:safety_filter}
    Consider the following least-restrictive intervention scheme $\phi$:
    \begin{equation}
        \phi(x_{0:k}, a_k) \!= \!\begin{dcases}
            a_k, & \Qopt{\psi}(x_{0:k}, a_k) \geq 0, \\
            \argmax_a \Qopt{\psi}(x_{0:k}, a), & \text{otherwise}.
        \end{dcases}
    \end{equation}
    Then, $\phi$ is a safety filter \cite{hsu2023safety}, i.e.,
    \begin{equation}
    \begin{split}
        \Qopt{\psi}(x_{0:k}, \pi_{[\psi]}(x_{0:k})) \geq 0 \\
        \implies 
        \Qopt{\psi}(x_{0:k}, \phi(x_{0:k}, a)) \geq 0, \quad \forall a \in \mathcal{A}.
    \end{split}
    \end{equation}
\end{thm}
\end{tcolorbox}

\begin{condproof}[thm:safety_filter]
    Suppose $\Qopt{\psi}(x_{0:k}, \pi^*(x_{0:k})) \geq 0$.
    We now split into the two cases in the definition of $\phi$.

    \noindent\textbf{Case 1: } If $\Qopt{\psi}(x_{0:k}, a_k) \geq 0$, then $\phi(x_{0:k}, a_k) = a_k$, so
    \begin{equation}
        \Qopt{\psi}(x_{0:k}, \phi(x_{0:k}, a_k)) = \Qopt{\psi}(x_{0:k}, a_k) \geq 0.
    \end{equation}

    \noindent\textbf{Case 2: } If $\Qopt{\psi}(x_{0:k}, a_k) < 0$, then $\phi(x_{0:k}, a_k) = \argmax_a \Qopt{\psi}(x_{0:k}, a)$, so
    \begin{align*}
        \Qopt{\psi}(x_{0:k}, \phi(x_{0:k}, a_k))
        &= \max_a \Qopt{\psi}(x_{0:k}, a) \\
        &\geq \Qopt{\psi}(x_{0:k}, \pi^*(x_{0:k})) \geq 0.
    \end{align*}

    In either case, $\Qopt{\psi}(x_{0:k}, \phi(x_{0:k}, a_k)) \geq 0$, which completes the proof.
\end{condproof}

\noindent This does not guarantee the system under the filtered controls satisfies $\psi$ for the same reasons as \Cref{ex:greedy_Q_suboptimal}.
To guarantee $\psi$, we need additional assumptions on $\psi$ that prevent it from indefinitely delegating to the future.
\begin{tcolorbox}[ThmFrame2]
\begin{thm} \label{thm:safety_filter_satisfy}
    Let $\psi$ be a TL formula such that, along any trajectory $x_{0:\infty}$, $\Vopt{\psi}(x_t) \geq 0$ for all $t \geq 0$ implies that $\rho_{[\psi]}(x_{0:\infty}) \geq 0$.
Then, any trajectory $\bar{x}_{0:\infty}$ generated by the system under the intervention scheme $\phi$ will satisfy the specification $\psi$ with non-negative robustness score, i.e., for any task policy $\tilde{\pi}$, we have $\rho_{[\psi]}(\bar{x}_{0:\infty}) \geq 0$, where
    \begin{equation}
        x_{t+1} = f(x_t, \phi(x_{0:t}, \tilde{\pi}(x_{0:t}))),\quad \forall t \geq 0.
    \end{equation}
\end{thm}
\end{tcolorbox}
\begin{tcolorbox}[AssmpFrame2]
\begin{proof}
    By the universal safety filter theorem \cite{hsu2023safety}, the safety filter $\phi$ guarantees that $\Vopt{\psi}(x_t) \geq 0$ for all $t \geq 0$. Applying the assumption then completes the proof.
\end{proof}
\end{tcolorbox}
\begin{tcolorbox}[ThmFrame2]
\begin{coroll}
    The assumption in \Cref{thm:safety_filter_satisfy} is guaranteed to be satisfied if $\psi$ does not contain $\LTLUntil$, or if all $\LTLUntil$ operators in $\psi$ are bounded, e.g., $\LTLUntil_{[0,N]}$ for some $N < \infty$.
\end{coroll}
\end{tcolorbox}

\section{Demonstrations}
\subsection{Two-Agent \texttt{GridWorld}: Filtering Collaboration}

To showcase the previous results, we solve a two-agent \texttt{GridWorld} problem (4-dim.) for a complex specification. Namely, we solve the optimal policy $\hat{\pi}$, dubbed $\pi_{\textbf{\texttt{spec}}}$, associated with the decomposed value \cite{sharpless2026bellman} for the specification below, and demonstrate (1) the ability of $\pi_{\textbf{\texttt{spec}}}$ to optimally guide the system towards satisfaction and (2) the ability of $\pi_{\textbf{\texttt{spec}}}$ to filter unsatisfactory nominal policies such that satisfaction is achieved regardless. In all scenarios, each agent is defined by a 2-dim. position and has actions $\{\leftarrow, \rightarrow, \uparrow, \downarrow, \emptyset\}$ which move it one space (or not at all). The value and policy are solved in the joint state-action space.

The \texttt{GridWorld} problem is defined by the specification
\begin{align} 
    \psi_{\textbf{\texttt{spec}}} := \mathsf{F}_{\small[0,60]}\textcolor{color_worksite}{r_{ws}} &\land (\lnot \textcolor{color_worksite}{r_{ws}} \mathsf{U} \textcolor{color_gear}{r_{g}}) \land \mathsf{FG}\textcolor{color_worksite}{r_{ws}} \land \: \label{eq:sim-spec} \\
    &\mathsf{G}(\mathsf{F}\textcolor{color_wood}{r_{wd}} \land \mathsf{F}\textcolor{color_saw}{r_{sw}}) \land \lnot \textcolor{color_door}{r_d}\mathsf{U}\textcolor{color_key}{r_k} \land \mathsf{G}\textcolor{color_safe}{r_{safe}}. \notag
\end{align}
This specification reads: Be in the worksite ($\textcolor{color_worksite}{r_{ws}}$ ) in $\mathbf{6 0}$ time units ($\mathsf{F}_{\small[0, 60]} \equiv \top \mathsf{U}_{\small[0, 60]}$ as in \Cref{sec:until_finite_horizon}) and eventually always be in the worksite but don't enter without gear ($\textcolor{color_gear}{r_{g}}$). Eventually, loop wood ($\textcolor{color_wood}{r_{wd}}$) and saw ($\textcolor{color_saw}{r_{sw}}$ ). One agent gets the key ($\textcolor{color_key}{r_{k}}$) to open the doors ($\textcolor{color_door}{r_{d}}$), don't hit walls, and keep exactly 2 units between agents to avoid collision and losing communication ($\textcolor{color_safe}{r_{safe}} := \lnot r_{walls} \land r_{\Vert\cdot\Vert_1 = 2}$).
We show a trajectory guided by $\pi_{\textbf{\texttt{spec}}}$ in the left of Fig.~\ref{fig:sim}.

Next, we demonstrate Thm.~\ref{thm:safety_filter_satisfy}, showing that $\pi_{\textbf{\texttt{spec}}}$ may serve as a safety filter for unsatisfactory user-defined policies, and do so in two cases: "\textbf{\textcolor{color_coffee}{coffee}}" and "\textbf{\textcolor{color_tea}{tea}}". In each case, we generate two nominal policies (via the same approach), in which agent 1 takes action for $\psi_1 :=\mathsf{F}\textcolor{color_coffee}{r_{\text{coffee}}} \land \mathsf{G} \lnot r_{walls}$ or $\psi_2 :=\mathsf{F}\textcolor{color_tea}{r_{\text{tea}}} \land \mathsf{G} \lnot r_{walls}$, and agent 2 takes action for $\psi_{\textbf{\texttt{spec}}}$, and the joint actions are filtered by $\psi_{\textbf{\texttt{spec}}}$.
The rollouts for the trajectory in either case can be seen in the center grids of Fig.~\ref{fig:sim}. Lastly, we compare this approach with the standard HJ/CBF approach, in which the filter is only defined by the safety spec. $\psi_{\textbf{\texttt{safe\_only}}} := \mathsf{G}\textcolor[rgb]{0.050,0.050,0.050}{r_{safe}}$. In this "baseline" case, despite agent 2 being defined by $\psi_{\textbf{\texttt{spec}}}$, without a safety filter for the joint system, deadlock arises between the two-agent policies and the system trajectory fails to satisfy $\psi_{\textbf{\texttt{spec}}}$, shown on the right side of Fig.~\ref{fig:sim}.

\subsection{Single-Agent Drone Safety Filter}

We demonstrate the execution of the policy $\hat{\pi}$ in a single-agent scenario with a Crazyflie drone, navigating obstacles towards a worksite (\Cref{fig:hw_rollouts}). The problem is modeled in the \texttt{GridWorld} dynamics, sending waypoints to the low-level controller that correspond to cell centers. Here, a reduced form of spec. \eqref{eq:sim-spec} is used, namely 
$\psi_{\textbf{\texttt{spec}}} := \mathsf{F}_{\small[0,50]}\textcolor{color_worksite}{r_{ws}} \land \mathsf{G}(\mathsf{F}\textcolor{color_wood}{r_{wd}} \land \mathsf{F}\textcolor{color_saw}{r_{sw}}) \land \lnot \textcolor{color_door}{r_d}\mathsf{U}\textcolor{color_key}{r_k} \land \mathsf{G}\textcolor{color_safe}{r_{safe}}$. We use this to filter a nominal policy which solely tries to reach the worksite $\psi:=\mathsf{F}\textcolor{color_worksite}{r_{ws}}$, and compare it with a filter with only the safety spec. $\psi_{\textbf{\texttt{safe\_only}}}$ and the unfiltered nominal policy (Fig.~\ref{fig:hw_rollouts}). Ultimately, the latter two get trapped or crash; only the trajectory filtered with $\psi_{\textbf{\texttt{spec}}}$ satisfies the spec.

\section{Conclusion}
This work constructs optimal policies for a large class of TL formulas by recursively applying the optimal policies for simpler subformulas.
We also show the optimal policy can be used as a safety filter to guarantee satisfaction of TL specifications under certain assumptions.
Future work includes expanding the solvable fragment to include more general conjunctions and nested temporal operators and investigating game-theoretic extensions. 

\printbibliography

\end{document}